\documentclass[nohyperref]{article}

\usepackage{microtype}
\usepackage{graphicx}
\usepackage{subfigure}
\usepackage{booktabs} 
\usepackage{algorithm}
\usepackage{tikz}
\usetikzlibrary{positioning}
\usepackage{booktabs,caption}
\usepackage[flushleft]{threeparttable}
\usetikzlibrary{arrows.meta}

\usepackage{hyperref}

\usepackage[accepted]{icml2022}

\usepackage{amsmath}
\usepackage{amssymb}
\usepackage{mathtools}
\usepackage{amsthm}

\usepackage[capitalize,noabbrev]{cleveref}

\def\shownotes{1}  \ifnum\shownotes=1
\newcommand{\authnote}[2]{{[#1: #2]}}
\else
\newcommand{\authnote}[2]{}
\fi

\newcommand{\wh}{\widehat}
\newcommand{\wt}{\widetilde}
\newcommand{\ov}{\overline}

\renewcommand{\tilde}{\wt}
\renewcommand{\hat}{\wh}
\renewcommand{\bar}{\ov}

\newcommand{\algnameacro}{\textsf{RPO-SAT}}
\newcommand{\algname}{Reference-based Policy Optimization with Stable at Any Time guarantee}
\newcommand{\propname}{Stable at Any Time}
\newcommand{\propnameacro}{SAT}

\theoremstyle{plain}
\newtheorem{theorem}{Theorem}[section]

\newtheorem{lemma}[theorem]{Lemma}

\theoremstyle{definition}
\newtheorem{definition}[theorem]{Definition}

\theoremstyle{remark}
\newtheorem{remark}[theorem]{Remark}

\usepackage[textsize=tiny]{todonotes}

\icmltitlerunning{Nearly Optimal Policy Optimization with Stable at Any Time Guarantee}

\begin{document}

\twocolumn[
\icmltitle{Nearly Optimal Policy Optimization with Stable at Any Time Guarantee}



\icmlsetsymbol{equal}{*}

\begin{icmlauthorlist}
\icmlauthor{Tianhao Wu}{equal,yyy}
\icmlauthor{Yunchang Yang}{equal,comp,pcl}
\icmlauthor{Han Zhong}{equal,comp}
\icmlauthor{Liwei Wang}{sch}
\icmlauthor{Simon S. Du}{xxx}
\icmlauthor{Jiantao Jiao}{yyy}
\end{icmlauthorlist}

\icmlaffiliation{yyy}{University of California, Berkeley}
\icmlaffiliation{comp}{Center for Data Science, Peking University}
\icmlaffiliation{sch}{Key Laboratory of Machine Perception, MOE,
School of Artificial Intelligence, Peking University}
\icmlaffiliation{xxx}{University of Washington}
\icmlaffiliation{pcl}{Peng Cheng Laboratory}

\icmlcorrespondingauthor{Tianhao Wu}{thw@berkeley.edu}
\icmlcorrespondingauthor{Yunchang Yang}{yangyc@pku.edu.cn}
\icmlcorrespondingauthor{Han Zhong}{hanzhong@stu.pku.edu.cn}

\icmlkeywords{Machine Learning, ICML, Reinforcement Learning, Policy Optimization}

\vskip 0.3in
]



\printAffiliationsAndNotice{\icmlEqualContribution} 

\begin{abstract}
Policy optimization methods are one of the most widely used classes of Reinforcement Learning (RL) algorithms. However, theoretical understanding of these methods remains  insufficient. Even in the episodic (time-inhomogeneous) tabular setting, the state-of-the-art theoretical result of policy-based method in \citet{shani2020optimistic} is only $\tilde{O}(\sqrt{S^2AH^4K})$ where $S$ is the number of states, $A$ is the number of actions, $H$ is the horizon, and $K$ is the number of episodes, and there is a $\sqrt{SH}$ gap compared with the information theoretic lower bound $\tilde{\Omega}(\sqrt{SAH^3K})$~\citep{jin2018q}. To bridge such a gap, we propose a novel algorithm Reference-based Policy Optimization with Stable at Any Time guarantee (\algnameacro), which features the property ``Stable at Any Time''. We prove that our algorithm achieves  $\tilde{O}(\sqrt{SAH^3K} + \sqrt{AH^4K})$ regret. When $S > H$, our algorithm is minimax optimal when ignoring logarithmic factors. To our best knowledge, RPO-SAT is the first computationally efficient, nearly minimax optimal policy-based algorithm for tabular RL.
\end{abstract}

\section{Introduction}
Reinforcement Learning (RL) has achieved phenomenal successes in solving complex sequential decision-making problems \citep{silver2016mastering,silver2017mastering,levine2016end,gu2017deep}. Most of these empirical successes are driven by policy-based (policy optimization) methods, such as policy gradient \citep{sutton1999policy}, natural policy gradient (NPG) \citep{kakade2001natural}, trust region policy optimization (TRPO) \citep{schulman2015trust}, and proximal policy optimization (PPO) \citep{schulman2017proximal}. For example, \citet{haarnoja2018soft} proposed a policy-based state-of-the-art reinforcement learning algorithm, soft actor-critic (SAC), which outperformed value-based methods in a variety of real world robotics tasks including manipulation and locomotion.
In fact, \citet{kalashnikov2018qt} observed that compared with value-based methods such as Q-learning, policy-based methods work better with dense reward. On the other hand, for sparse reward cases in robotics, value-based methods perform better.

Motivated by this, a line of recent work \citep{fazel2018global,bhandari2019global,liu2019neural,wang2019neural,agarwal2021theory} provides global convergence guarantees for these popular policy-based methods. However, to achieve this goal, they made several assumptions. \citet{agarwal2021theory} assumes they have the access to either the exact population policy gradient or an estimate of it up to certain precision for all states uniformly compared with the state distribution induced by $\pi^*$, bypassing the hardness of exploration. They showed that even with this stringent assumption, the convergence rate would depend on the distribution mismatch coefficient $D_{\infty} = \max_s \big(\frac{d_{s_0}^{\pi^*}(s)}{\mu(s)}\big)$, where $\mu$ is the starting state distribution of the algorithm and $d_{s_0}^{\pi^*}(s)$ is the stationary state distribution of the optimal policy $\pi^*$ starting from $s_0$. This dependency is problematic since $D_{\infty}$ is small only when the initial distribution has a \emph{good coverage} of the optimal stationary distribution, which may not happen in practice. 

However, in online value-based RL, algorithm such as \citet{azar2017minimax} can achieve fast convergence rate (or regret) independent of the distribution mismatch coefficient, or equivalently, without the \emph{coverage} assumption. Though value-based methods have achieved the information theoretical optimal regret in tabular \citep{azar2017minimax} and linear MDPs settings \citep{zanette2020learning}, it remains unclear whether policy-based methods can achieve information theoretically optimal regret in the same settings. To address this issue, \citet{cai2020provably} proposed the idea of optimism in policy optimization, which seems similar to the value-based optimism but different in nature, since it encourages optimism for $Q^{\pi}$ instead of $Q^{*}$ (Section \ref{sec:technique}). With this new idea, \citet{cai2020provably,shani2020optimistic} managed to establish the regret guarantees without additional assumptions, though the regret is suboptimal.
 
In this work, we focus on the same setting as in \citet{shani2020optimistic}: episodic \emph{tabular} MDPs with \emph{unknown transitions}, \emph{stochastic rewards/losses}, and \emph{bandit feedback}. In this setting, the state-of-the-art result of policy-based method is $\tilde{O}(\sqrt{S^2AH^4K})$~\citep{shani2020optimistic}. Here $S$ and $A$ are the cardinality of states and actions, respectively, $H$ is the episode horizon, and $K$ is the number of episodes. Compared with the information theoretic limit \citep{jaksch2010near,azar2017minimax,jin2018q,domingues2021episodic}, there is still a gap $\sqrt{SH}$. 
\subsection{Main Contributions}
In this paper we present a novel provably efficient policy optimization algorithm, \emph{\algname}\,(\algnameacro).
We establish a high probablity regret upper bound $\tilde{O}(\sqrt{SAH^3K}+\sqrt{AH^4K})$ for our algorithm. Importantly, if $S> H$, the main term in this bound matches the information theoretic limit $\Omega(\sqrt{SAH^3K})$ \citep{jin2018q}, up to some lower order term $\tilde{O}(\operatorname{poly}(S,A,H)K^{1/4})$. We introduce our algorithmic innovations and analytical innovations as follows:

\textbf{Algorithmic innovations:} 
\begin{itemize}
    \item We introduce a novel reference $V$ estimator in our algorithm. It is conceptually simple and easy to implement as it just updates the reference value to be the mean of empirical $V$ values when some conditions are triggered (cf. Algorithm \ref{alg:pore} Lines 18-20).
    \item We carefully incorporate the reference $V$ estimator into our bonus term, in the way of adding a weighted absolute difference between the estimated $V$ values and the reference $V$ values to control the instability of the estimation process. 
    Readers may refer to Section \ref{sec:technique} for more details.
    \item Another highlight is that we modify the policy improvement phase of our algorithm to meet a novel property, which we called \textbf{\propname\,(\propnameacro)}. (For a detailed definition see Equation \eqref{eq:tech:1} in Section~\ref{sec:technique}.) More specifically, instead of using the KL-divergence regularization term proposed in \citet{shani2020optimistic}, we use the $\ell_2$ regularization term. This is crucial to ensure \propnameacro.  
    \end{itemize}

\textbf{Analytical innovations:}
\begin{itemize}
    \item We prove that our algorithm satisfies the \propnameacro\,property. The analysis is done by two steps: First, we establish a \emph{1}-st step regret bound $\tilde{O}(\sqrt{S^2AH^4K})$. Second, we use a new technique ``\emph{Forward Induction}'' to prove the same for all the $h$-th step regret. Here the $h$-step regret is defined in \eqref{eq:def:regret}. Readers may refer to Section \ref{sec:technique} for more details of the ``Forward Induction'' technique. 
    \item We show that the combination of the \propnameacro\;property and the simple reference $V$ estimator yields a precise approximation of $V^*$. We use this property to derive a $\tilde{O}(\sqrt{SAH^3K})$ upper bound for the sum of bonus terms, which leads to a $\sqrt{S}$ reduction in terms of regret.
\end{itemize}

\subsection{Related Work} \label{sec:related:work}

Our work contributes to the theoretical investigations of policy-based methods in RL \citep{cai2020provably,shani2020optimistic,lancewicki2020learning,fei2020dynamic,he2021nearly,zhong2021optimistic,luo2021policy,zanette2021cautiously}. The most related policy-based method is proposed by \citet{shani2020optimistic}, who also studies the episodic tabular MDPs with {unknown transitions}, {stochastic losses}, and {bandit feedback}. It is important to understand whether it is possible to eliminate this gap and thus achieve minimax optimality, or alternatively to show that this gap is inevitable for policy-based methods. 

We also provide interesting practical insights. For example, the usage of reference estimator and $\ell_2$ regularization with decreasing learning rate to stabilize the estimate of $V$ and $Q$. The use of reference estimator can be traced to \cite{zhang2020almost}. They use the reference estimator to maximize the data utilization, hence reduce the estimation variance. However, our usage of reference estimator is different from theirs. The reason why they can reduce a $\sqrt{H}$ factor is because the bottleneck term is only estimated using $1/H$ fraction of data, while the usage of reference can make use of all the data, hence fully utilizing available data. However, for policy-based RL, there is no such problem of data utilization. In fact, the bottleneck of policy-based methods is the \emph{instability} of $Q$ estimation, therefore we use the reference estimator to stabilize the estimation process and reduce a $\sqrt{S}$ factor in the regret. Readers may turn to Section \ref{sec:technique} for a detailed explanation for ``instability''.

Our work is also closely related to another line of work on value-based methods. In particular, \citet{azar2017minimax,zanette2019tighter,zhang2020reinforcement,zhang2020almost,menard2021ucb} have shown that the value-based methods can achieve $\tilde{O}(\sqrt{SAH^3K})$ regret upper bound, which matches the information theoretic limit. Different from these works, we are the first to prove the (nearly) optimal regret bound for policy-based methods.

\section{Preliminaries}
A finite horizon stochastic Markov Decision Process (MDP) with time-variant transitions $\mathcal{M}$ is defined by a tuple $(\mathcal{S}, \mathcal{A}, H, P = \left\{P_{h}\right\}_{h=1}^{H}, c = \left\{c_{h}\right\}_{h=1}^{H})$, where $\mathcal{S}$ and $\mathcal{A}$ are finite state and action spaces with cardinality $S$ and $A$, respectively, and $H \in \mathbb{N}$ is the horizon of the MDP. At time step $h$, and state $s$, the agent performs an action $a$, transitions to the next state $s^{\prime}$ with probability $P_{h}(s^{\prime} \mid s, a)$, and suffers a random cost $C_{h}(s, a) \in[0,1]$ drawn i.i.d from a distribution with expectation $c_{h}(s, a)$. 

A stochastic policy $\pi: \mathcal{S} \times[H] \rightarrow \Delta_{A}$ is a mapping from states and time-step indices to a distribution over actions, i.e., $\Delta_{A}=\left\{\pi \in \mathbb{R}^{A}: \sum_{a} \pi(a)=1, \pi(a) \geq 0\right\} .$ The performance of a policy $\pi$ when starting from state $s$ at time $h$ is measured by its value function, which is defined as 
\begin{equation}
V_{h}^{\pi}(s)=\mathbb{E}\bigg[\sum_{h^{\prime}=h}^{H} c_{h^{\prime}}\left(s_{h^{\prime}}, a_{h^{\prime}}\right) \mid s_{h}=s, \pi, P\bigg].
\end{equation}
The expectation is taken with respect to the randomness of the transition, the cost function and the policy. The $Q$-function of a policy given the state action pair $(s, a)$ at time-step $h$ is defined by
\begin{equation}
Q_{h}^{\pi}(s, a)=\mathbb{E}\bigg[\sum_{h^{\prime}=h}^{H} c_{h^{\prime}}(s_{h^{\prime}}, a_{h^{\prime}}) \mid s_{h}=s, a_{h}=a, \pi, P\bigg].
\end{equation}
By the above definitions, for any fixed policy $\pi$, we can obtain the Bellman equation 
\begin{equation}
\begin{aligned}
Q_{h}^{\pi}(s, a)&=c_{h}(s, a)+P_{h}(\cdot \mid s, a) V_{h+1}^{\pi}(\cdot),\\
V_{h}^{\pi}(s)&=\langle Q_{h}^{\pi}(s, \cdot), \pi_{h}(\cdot \mid s)\rangle.
\end{aligned}
\end{equation}
An optimal policy $\pi^{*}$ minimizes the value for all states $s$ and time-steps $h$ simultaneously, and its corresponding optimal value is denoted by $V_{h}^{*}(s)=$ $\min _{\pi} V_{h}^{\pi}(s)$, for all $h \in[H] .$ We consider an agent that repeatedly interacts with an MDP in a sequence of $K$ episodes such that the starting state at the $k$-th episode, $s_{1}^{k}$, is initialized by a fixed state $s_{1}$\footnote{Our subsequent analysis can be extended to the setting where the initial state is sampled from a fixed distribution.}.

In this paper we define the notion of \emph{$h$-th step regret}, $\operatorname{Regret}_h$, as follows:
\begin{align} \label{eq:def:regret}
\operatorname{Regret}_h(K)=\sum_{k=1}^{K} \big(V_{h}^{\pi^{k}}(s_{h}^{k})-V_{h}^{*}(s_{h}^{k})\big).
\end{align}
When $h=1$, this matches the traditional definition of regret, which measures the performance of the agent starting from $s_1$. In this case we also use $\operatorname{Regret}(K)$ for simplicity.

\paragraph{Notations and Definitions} We denote the number of times that the agent has visited state $s$, state-action pair $(s, a)$ and state-action-transition pair $(s,a,s^{\prime}$ at the $h$-th step by $n_{h}^{k}(s)$, $n_{h}^{k}(s, a)$ and $n_{h}^{k}(s, a, s^{\prime})$ respectively. We denote by $\bar{X}_{k}$ the empirical average of a random variable $X$. All quantities are based on experience gathered until the end of the $k$-th episode. We denote by $\Delta_{\mathcal{A}}$ the probability simplex over the action space, i.e. $\Delta_{\mathcal{A}}=\{(p_1,...,p_{|\mathcal{A}|})\mid p_i\geq 0, \sum_{i}p_i=1\}$.

We use $\tilde{O}(X)$ to refer to a quantity that depends on $X$ up to a poly-log expression of a quantity at most polynomial in $S, A, K, H$ and $\delta^{-1} .$ Similarly, $\lesssim$ represents $\leq$ up to numerical constants or poly-log factors. We define $X \vee Y:=\max \{X, Y\}$.

\section{\algnameacro: \algname}
\begin{algorithm}[H]
    \caption{\algname\,(\algnameacro)}\label{alg:pore}
    \begin{algorithmic}[1]
    \STATE initialize $Q_{h}(x, a) \leftarrow 0$, $V_h(x,a)\leftarrow 0$ and $ V_h^{\mathrm{ref}}(x,a) \leftarrow 0$, $C_0=\sqrt{S^3AH^3}$
    \FOR{episode $k=1,...,K$}
    \STATE Rollout a trajectory by acting $\pi^k$
    \STATE Update counters and empirical model $n^k = \{n_h^k\}_{h \in [H]}, \bar{c}^k = \{\bar{c}_h^k\}_{h \in [H]}, \bar{P}^k = \{\bar{P}_h^k\}_{h \in [H]}$ 
    \FOR{Step $h=H,...,1$}
    \FOR{$\forall s,a\in \mathcal{S}\times\mathcal{A}$}
    \STATE Calculate $u_h^k$ as in \eqref{eqn:longlong}
    \STATE $b^k_h(s,a) = \min\{u_{h}^{k}(s,a), \sqrt{\frac{2 \ln \frac{2 S A H T}{\delta^{\prime}}}{n_{h}^{k}(s, a) }}+H\sqrt{\frac{4 S \ln \frac{3 S A H T}{\delta^{\prime}}}{n_{h}^{k}(s, a) }}\}$\label{line:ref:8}
    \STATE $Q^k_h(s,a) = \max\{ \bar{c}_h^k(s,a)+\bar{P}^{k}_h(\cdot \mid s,a) V^k_{h+1}(\cdot) -b^k_h(s,a), 0\}$
    \FOR{$\forall s\in \mathcal{S}$}
    \STATE $V^k_h(s) = \langle Q^k_h(s,\cdot),\pi^k_h(\cdot|s)\rangle$
    \ENDFOR
    \ENDFOR
    \ENDFOR
    \FOR{$\forall h,s,a\in[H]\times\mathcal{S}\times\mathcal{A}$}
    \STATE $\pi^{k+1}_h(\cdot|s) = \text{argmin}_{\pi_h} \eta_k\langle Q^k_h(\cdot,s),\pi_h -\pi^k_h\rangle+ \|\pi_h-\pi_h^k\|_2^2$
    \ENDFOR
    \FOR {$\forall h,s,a\in[H]\times\mathcal{S}\times\mathcal{A}$ such that $n_h^k(s) \geq C_0\sqrt{k}$} \label{line:ref:1}
    \STATE $V^{\mathrm{ref}}_h(s) = \frac{1}{n^k_h(s)}\sum_{i=1}^k V_h^i(s^i_h)\textbf{1}[s^i_h=s] $ \label{line:ref}
    \ENDFOR \label{line:ref:2}
    \ENDFOR
    \end{algorithmic}
    \end{algorithm}

In this section, we present our algorithm \algnameacro\,(\algname). The pseudocode is given in Algorithm \ref{alg:pore}.

We start by reviewing the optimistic policy optimization algorithms OPPO and POMD proposed by \citet{cai2020provably,shani2020optimistic}. In each update, OPPO and POMD involve a policy improvement phase and a policy evaluation phase. In policy evaluation phase, OPPO and POMD explicitly incorporate a UCB bonus function into the estimated $Q$-function to promote exploration. Then, in the policy improvement phase, OPPO and POMD improve the policy by Online Mirror Descent (OMD) with KL-regularization, where the estimated $Q$-function serves as the gradient. Compared with existing optimistic policy optimization algorithms in \citet{shani2020optimistic,cai2020provably}, our algorithm has three novelties.

First, in policy evaluation phase, we introduce the reference $V$ estimator (Lines 18-20). Specifically, for any $s \in \cal{S}$, if the number of visiting $s$ satisfies the condition in Line 18, we update the reference value to be the empirical mean estimator of $V$ (Line 19). \citet{zhang2020almost} also adopts a reference value estimator. Nevertheless their update conditions and methods are different from us, and they reduce a $\sqrt{H}$ factor while we reduce a $\sqrt{S}$ factor.

Second, we make some modifications in policy improvement phase. Specifically, the policy optimization step in the $h$-th step of the $k$-th episode is  
\begin{align*}
    \pi^{k+1}_h(\cdot|s) = \text{argmin}_{\pi_h} \eta_k\langle Q^k_h(\cdot,s),\pi_h -\pi^k_h\rangle+ D(\pi_h, \pi_h^k),
\end{align*}
where $\eta_k$ is the stepsize in the $k$-th episode, $D$ is some distance measure and $Q_h^k$ is the estimated $Q$-function in the $h$-th step of the $k$-th episode. 
Different from \citet{shani2020optimistic,cai2020provably} choosing $D(\pi_h, \pi_h^k) = \mathrm{KL}(\pi_h, \pi_h^k)$, we choose $D(\pi_h, \pi_h^k) = \|\pi_h - \pi_h^k\|_2^2$. In this case the solution of the OMD is: 
\begin{equation*}
    \pi^{k+1}_h(\cdot|s) =\Pi_{\Delta_{\mathcal{A}}}\left(\pi^k_h(\cdot|s)-\eta_{k} Q^k_h(\cdot,s)\right)
\end{equation*}
where $\Pi_{\Delta_{\mathcal{A}}}$ is the Euclidean projection onto $\Delta_{\mathcal{A}}$, i.e. $\Pi_{\Delta_{\mathcal{A}}}(\boldsymbol{x})=\operatorname{argmin}_{\boldsymbol{y} \in \Delta_{\mathcal{A}}}\|\boldsymbol{x}-\boldsymbol{y}\|_{2}$. Unlike previous works, we also adopt a decreasing learning rate schedule instead of a fixed learning rate. These modifications are necessary for our analysis since they ensure the \propnameacro\;property, making it possible to learn a good reference $V$ value.

Finally, we design a novel bonus term which carefully incorporate the reference $V$ estimator mentioned before by adding a weighted absolute difference Term (iii), which is also referred to as the \emph{instability} term. Specifically, we define
\begin{align}\label{eqn:longlong}
u_{h}^{k}(s,a)=\text{Term}(i)+\text{Term}(ii)+\text{Term}(iii)+\text{Term}(iv)
\end{align}
where
\begin{align*}
\text{Term}(i)&=\sqrt{\frac{2 \ln \frac{2 S A H T}{\delta^{\prime}}}{n_{h}^{k}(s, a)  }}, \notag\\
\text{Term}(ii)&=\sqrt{\frac{6 \mathbb{V}_{Y \sim \bar{P}_h^k (\cdot \mid s, a)}V_{h+1}^{\mathrm{ref}}(Y)\ln \left(\frac{2 SAHK}{\delta^{\prime}}\right)}{n_{h}^{k}(s, a)}}+\\
&\sqrt{\frac{4H \ln \left(\frac{2 SAHK}{\delta^{\prime}}\right)}{S \cdot n_{h}^{k}(s, a)}}+\frac{8\sqrt{SH^2} \ln \left(\frac{2 SAHK}{\delta^{\prime}}\right)}{3 n_{h}^{k}(s, a)}, \notag\\
\text{Term}(iii)&\!=\!\sum_{y\in\mathcal{S}}\!\left(\!\sqrt{\frac{2\!\bar{P}_{h}^k(\!y\! \mid\! s\!,\! a\!)(\!1\!-\!\bar{P}_{h}^k\!(\!y \!\mid\! s,\! a)) \!\ln\! \left(\!\frac{2 SAHK}{\delta^{\prime}}\!\right)}{n_{h}^{k}(s, a)-1}}\right.\\
&\left.+\frac{7 \ln \left(\frac{2 SAHK}{\delta^{\prime}}\right)}{3 n_{h}^{k}(s, a)}\right)|V^k_{h+1}(y)-V^{\mathrm{ref}}_{h+1}(y)|, \notag\\
\text{Term}(iv)&=\sqrt{\frac{4 H\ln \frac{3 S A H T}{\delta^{\prime}}}{n_{h}^{k}(s, a)  }}+\frac{7SH \ln \left(\frac{2 SAHK}{\delta^{\prime}}\right)}{3 n_{h}^{k}(s, a)}\\
&+\frac{2S^{3/2}A^{1/4}H^{7/4}K^{1/4}\sqrt{\ln \left(\frac{2 SAHK}{\delta^{\prime}}\right)}}{n_h^{k}(s,a)}.
\end{align*}
where $(i)$ is the estimation error for reward functions and $(ii)$-$(iv)$ are estimation errors for transition kernels. We will provide more explanations for this seemingly complicated term in Section \ref{sec:technique}. Furthermore, we set the bonus as 
\begin{align} \label{eq:def:bonus}
b^k_h(s,a) =& \min\left\{u_{h}^{k}(s,a),\right.\notag\\
&\left.\sqrt{\frac{2 \ln \frac{2 S A H T}{\delta^{\prime}}}{n_{h}^{k}(s, a) }}+H\sqrt{\frac{4 S \ln \frac{3 S A H T}{\delta^{\prime}}}{n_{h}^{k}(s, a) }}\right\}.
\end{align}
With this carefully chosen bonus function, we can also achieve optimism like previous work in optimistic policy optimization \citep{shani2020optimistic,cai2020provably}. Notably, our bonus function is smaller than that in \citet{shani2020optimistic}, thus leads to our tighter regret bound.

Now we state our main theoretical results for \algnameacro. 
\begin{theorem}[Regret bound] \label{thm:fine_reg}
Suppose in Algorithm \ref{alg:pore}, we choose $\eta_t = O(\sqrt{1/(H^2At)})$, and $C_0=\sqrt{S^3AH^3}$, then for sufficiently large $K$, we have
\begin{align*}
    \operatorname{Regret}(K) \le &\tilde{O}(\sqrt{SAH^3K} + \sqrt{AH^4K}+\\
    &S^{5/2}A^{5/4}H^{3/2}K^{1/4}).
\end{align*}
\end{theorem}
We provide a proof sketch in Section \ref{sec:proof:sketch}. The full proof is in appendix \ref{appendix:full:proof}. Note that previous literature shows that the regret lower bound is $\Omega(\sqrt{SAH^3K})$ \citep{jin2018q}. Hence our result matches the information theoretic limit up to logarithmic factors when $S > H$.

\section{Technique Overview} \label{sec:technique}

In this section, we illustrate the main steps of achieving near optimal regret bound and introduce our key techniques.

\vskip 4pt
\noindent{\bf Achieving Optimism via Reference.} 
Similar to previous works \citep{cai2020provably,shani2020optimistic}, to achieve optimism, a crucial step is to design proper bonus term to upper bound $(\bar{P}_{h}^k-P_{h})(\cdot \mid s,a)V^{k}_{h+1}(\cdot)$. For example, \citet{jaksch2010near,shani2020optimistic} bound this term in a separate way:
\begin{align} \label{eqn:naivebonus}
    &\left|(\bar{P}_{h}^k - P_{h})(\cdot \mid s,a)V^{k}_{h+1}(\cdot)\right| \\
    \leq& \|(\bar{P}_{h}^k - P_{h})(\cdot \mid s,a)\|_{1} \cdot \|V^{k}_{h+1}(\cdot)\|_{\infty} \le \tilde{O}\Big(\sqrt{\frac{SH^2}{n_h^k(s, a)}}\Big), \notag
\end{align}
which will leads to an additional $\sqrt{S}$ factor due to the absence of making use of the optimism of $Q^k$. This issue is later addressed by value-based algorithm UCBVI \citep{azar2017minimax}. They divide the term into two separate terms: 
\begin{align} \label{eqn:desiredbonus}
    &(\bar{P}_{h}-P_{h})(\cdot \mid s,a)V^{k}_{h+1}(\cdot)=(\bar{P}_{h}^k-P_{h})(\cdot\mid s,a)V^{*}_{h+1}(\cdot)\notag\\&+(\bar{P}_{h}^k-P_{h})(\cdot \mid s,a)(V^{k}_{h+1} - V_{h+1}^*)(\cdot).
\end{align}
They bound the first term using straightforward application of Chernoff-Hoeffding inequality, which removes the $\sqrt{S}$ factor since $V^{*}$ is deterministic. Thanks to the fact that $V_h^k \le V_h^*$ for any $(k, h) \in [K]\times [H]$, they can bound the second term successfully (see Appendix \ref{sec:explanation} for more details). However, this approach is not applicable for policy-based methods which improve the policy in a conservative way instead of choosing the greedy policy (when the stepsize $\eta \rightarrow \infty$, the OMD update becomes the greedy policy, and for any $\eta < \infty$ this update is ``conservatively'' greedy). This key property of policy-based methods makes it only possible to ensure $V^k_h\leq V^{\pi^k}_h$, the optimism $V^k_h\leq V^{*}_h$ doesn't hold in general. To tackle this challenge, we notice that as long as $V^k_h$ converges to $V_h^*$ sufficiently fast (at least on average), $(\bar{P}_{h}^k-P_{h})(\cdot \mid s,a)(V^{k}_{h+1} - V_{h+1}^*)(\cdot)$ can be bounded by a term related to the rate of convergence. This leads to the important notion called \textbf{\propname\,(\propnameacro)}. Precisely, we introduce the \emph{reference V estimator} $V^{\mathrm{ref}} = \{V^{\mathrm{ref}}_h\}_{h \in [H]}$ and decompose $(\bar{P}_{h}^k-P_{h})(\cdot \mid s,a)V^{k}_{h+1}(\cdot)$ as 
\begin{align} \label{eqn:desiredbonus2}
    &(\bar{P}_{h}^k-P_{h})(\cdot \mid s,a)V^{k}_{h+1}(\cdot)\notag\\ =&\underbrace{(\bar{P}_{h}^k-P_{h})(\cdot\mid s,a)V^{*}_{h+1}(\cdot)}_{(a)}\notag\\
    & + \underbrace{(\bar{P}_{h}^k-P_{h})(\cdot \mid s,a)(V^{k}_{h+1} - V_{h+1}^{\mathrm{ref}})(\cdot)}_{(b)} \notag\\
    & +  \underbrace{(\bar{P}_{h}^k-P_{h})(\cdot \mid s,a)(V^{\mathrm{ref}}_{h+1} - V_{h+1}^{*})(\cdot)}_{(c)}.
\end{align}
By standard concentration inequalities, we can bound the first term by the quantity depending on the variance of $V_{h+1}^*$. Once we have \propnameacro, an easy implication is that $|V_{h}^{\mathrm{ref}}(s)-V_{h}^{*}(s)|\leq O(\sqrt{H/S})$ for any $n^k_h(s)$ large enough (Lemma \ref{lem:average:convergence}). We can replace the variance of $V_{h+1}^*$ (unknown) by the variance of $V_{h+1}^{\mathrm{ref}}$ (known) and bound the second and third term in an analogous way of \eqref{eqn:naivebonus} and remove the factor $\sqrt{S}$ as desired. Specifically, we can upper bound Terms $(a), (b), (c)$ in \eqref{eqn:desiredbonus2} by Terms $(ii), (iii), (iv)$ in \eqref{eqn:longlong}, respectively. 

\vskip 4pt
\noindent{\bf \propname.} The high-level idea is that in order to control $(\bar{P}_{h}^k-P_{h})(\cdot \mid s,a)(V^{k}_{h+1} - V_{h+1}^*)(\cdot)$, term $V^{k}_{h+1} - V_{h+1}^*$ must satisfy some properties. For example, \citet{zhang2020almost} shows that in the senario of greedy policy, this term converges to zero as the number of visit goes to infinity. However, due to the nature of conservative policy update scheme, we cannot guarantee the convergence, unless we make the \emph{coverage} assumption as in \citet{agarwal2021theory}. Fortunately, it's possible to obtain an average convergence guarantee, which we called \propnameacro. Specifically, we say that an algorithm satisfy the property of \propnameacro\,if for all $K'$ and $h$,
\begin{align} \label{eq:tech:1}
    \sum_{k = 1}^{K'} |V_h^k(s^k_h) - V_h^*(s^k_h)| \le O(\sqrt{K'}).
\end{align}    
The meaning of \textbf{\propname}\;can be interpreted as follows: The above inequality implies that the estimation $V^k_h$ varies around the fixed value $V_h^*$, hence \textbf{``stable''}. And since we impose that the inequality holds for all $h$ and $K'$, hence \textbf{``at any time''}. For this reason, we also name $\sum_{y\in\mathcal{S}}\sqrt{\frac{\bar{P}^k_h(y|s,a)}{n^k_h(s,a)}}|V^k_{h+1}(y)-V^{\mathrm{ref}}_{h+1}(y)|$ in the bonus term (\ref{eqn:longlong}) as the \textbf{instability} term, since it's an upper bound of $(\bar{P}_{h}^k-P_{h})(\cdot \mid s,a)(V^{k}_{h+1} - V_{h+1}^{\mathrm{ref}})(\cdot)$, which measures the instability of $V^k$ with respect to the fixed reference $V^{\mathrm{ref}}$. 

Another nice implication is that we can obtain a precise estimate of $V^*$ if \propnameacro\;is satisfied. Specifically, combining the update rules of reference $V$ estimator and the fact that $n^k_h(s)\geq C_0\sqrt{k}$, we have:
\begin{align*}
&|V^{\mathrm{ref}}_h(s)-V^*_h(s)| \\
\leq &\frac{1}{n^k_h(s)}\sum_{i=1}^k |V_h^i(s^i_h)-V^*_h(s^i_h)|\textbf{1}[s^i_h=s] \leq O(1/C_0).
\end{align*}
\vskip 4pt
\noindent{\bf Coarse Regret Bound.} We are able to show that \eqref{eq:tech:1} can be implied by  
\begin{align}
    \text{Regret}_1(K') \le O(\sqrt{K'}).
\end{align}
for any $K' \in [K]$. See Lemma \ref{lem:average:convergence} for details. We point out that establishing Equation (\ref{eq:tech:1}) from coarse regret bound is highly non-trivial. There exists two challenges: 
\begin{enumerate}
\item Previous policy-based methods \citep{cai2020provably,shani2020optimistic} choose a constant mirror descent stepsize depending on $K$, which impedes us from obtaining regret bound sublinear on $K'$ for all $K' \in [K]$.

\item For the step $2 \le h \le H$, the state $s_h^k$ is not fixed, which makes the OMD term difficult to bound. 
\end{enumerate}
To this end, we use the following two novel techniques: 
\begin{enumerate}
\item Replace the unbounded KL-divergence regularization term by the bounded $\ell_2$ regularization term, which allows us to choose varying stepsize which depends on the current time step instead of $K$.

\item \emph{Forward Induction}, that is, deriving the regret upper bound for the \emph{1}-st step and obtaining the regret upper bounds for $h$-th ($2 \le h \le H$) step by induction. See Section \ref{sec:proof:sketch} for more details.
\end{enumerate}
Putting these together, we can remove the additional factor $\sqrt{S}$ as desired. If we also adopt the Bernstein variance reduction technique \citep{azar2017minimax}, we can further improve a $\sqrt{H}$ factor in the main term.

\section{Proof Sketch of Theorem \ref{thm:fine_reg}} \label{sec:proof:sketch}

In this section, our goal is to provide a sketch of the proof of Theorem \ref{thm:fine_reg}. One key ingredient in regret analysis in RL is the optimism condition, namely
\begin{align} \label{eqn:sketch_opt}
    -2 b^k_h(s,a)\!\leq \! Q^k_h(s,a)\!-\!c_h(s,a)\!-\!P_h(\cdot \mid s,a)V^k_{h+1}(\cdot)\!\leq\! 0.
\end{align}
This condition guarantees that $Q^k_h$ is an optimistic estimation of $Q_h^{\pi^k}$ (or $Q_h^{*}$), and that the regret can be bounded by the sum of bonus terms \citep{jaksch2010near,azar2017minimax,cai2020provably,shani2020optimistic}. 

In our case, it does not appear straightforward to show such condition hold or not since the bonus term $b_h^k$ depends on the value of $V^{\mathrm{ref}}$, which we do not know in advance. Dealing with this problem needs additional effort. But let's put it aside for a while, and assumes that \eqref{eqn:sketch_opt} holds temporarily. In Section \ref{sec:warmup}, we first demonstrate the intuition of the proof under this assumption. Then in Section \ref{sec:howto}, we show how to remove this assumption, with additional techniques.

\subsection{A warm-up: the Optimism Assumption}\label{sec:warmup}
In this section, we first demonstrate the intuition of the proof, under the assumption that optimism condition \eqref{eqn:sketch_opt} holds. 
We first recall the useful notion of $h$-th step regret in episode $K'$ as follows:
$$\text{Regret}_h(K') =\sum^{K'}_{k=1} \big(V_h^{\pi^k}(s^k_h)-V_h^*(s^k_h)\big). $$
The notion aligns with the normal regret if $h = 1$, in other words, we have $\text{Regret}_1(K') = \text{Regret}(K')$.
We first state the following lemma, which bounds the estimation error using bonus function:
\begin{lemma}[Bounding estimation error] \label{lem:regret:bound}
Suppose that optimism holds, namely for $\forall k\leq K, h\leq H, s\in\mathcal{S}, a\in\mathcal{A}$,
\begin{align*}
    -2 b^k_h(s,a)\leq Q^k_h(s,a)-c_h(s,a)-P_h(\cdot|s,a)V^k_{h+1}\leq 0.
\end{align*}
Then for $\forall K'\leq K, h' \leq H$, it holds with high probability that
\begin{align*}
    &\sum_{k=1}^{K'}\big(V^{\pi^k}_h(s^k_h)-V^k_{h}(s^k_h)\big)\\
    \leq& \tilde{O}\big(\sum_{k=1}^{K'}\sum_{h=h'}^H b^k_h(s^k_h,a^k_h)\big)+\tilde{O}(\sqrt{H^3K'}).
\end{align*}\label{lem:estimation_error}
\end{lemma}
\begin{proof}
    See \S \ref{appendix:pf:lem:regret:bound} for a detailed proof.
\end{proof}
The above lemma is useful since it tells us that the sum of the estimation error $V^{\pi^k}_h(s^k_h)-V^k_h(s^k_h)$ can be roughly viewed as the sum of bonus function. By the choice of our bonus term, we have $b^k_h(s,a)\leq \sqrt{\frac{SH^2}{n^k(s,a)}}$. Following standard techniques from \cite{shani2020optimistic}, we get a easy corollary $\sum_{k=1}^{K'}\big(V^{\pi^k}_h(s^k_h)-V^k_{h}(s^k_h)\big)\leq \tilde{O}(\sqrt{S^2AH^4K'})$.

Next, we show that our algorithm satisfies a coarse regret bound:
\begin{lemma}[Coarse regret bound]
With the same assumptions and notations in Lemma \ref{lem:estimation_error}, we have for $\forall K'\leq K$:
\begin{equation}\label{lem:coarse_regret}
\operatorname{Regret}(K') \!= \!\sum_{k=1}^{K'} \big(V^{\pi^k}_1(s_1^k)-V^*_1(s_1^k)\big)\!\leq\! \tilde{O}(\sqrt{S^2AH^4K'}).
\end{equation}
\end{lemma}
\begin{proof}
We first decompose the regret term in the following way:
\begin{align}
    &\sum_{k=1}^{K'} \big(V_{1}^{\pi^{k}}\left(s_{1}\right)-V_{1}^{*}\left(s_{1}\right)\big)\notag\\
    =&\underbrace{\sum_{k=1}^{K'} \big(V_{1}^{\pi^{k}}\left(s_{1}\right)-V_{1}^{k}\left(s_{1}\right)\big)}_{(i)}+\underbrace{\sum_{k=1}^{K'}\big(V_{1}^{k}\left(s_{1}\right)-V_{1}^{*}\left(s_{1}\right)\big)}_{(ii)}\label{eqn:reg_decomp}
\end{align}
$\text{Term}\ (i)$ can be bounded by $\tilde{O}(\sqrt{S^2AH^4K'})$ using the corollary of Lemma \ref{lem:estimation_error}. We take a closer look at $\text{Term}\ (ii)$, by standard regret decomposition lemma (Lemma \ref{lemma:regret:decomposition}):
\begin{align*}
    &\text{Term}\ (ii)=\\
    &\underbrace{\sum_{k,h} \mathbb{E}\left[\left\langle Q_{h}^{k}\left(s_{h}, \cdot\right), \pi_{h}^{k}\left(\cdot \mid s_{h}\right)-\pi_{h}^*\left(\cdot \mid s_{h}\right)\right\rangle \mid s_{1}, \pi^*, P\right]}_{(iii)}\\
    &+\!\underbrace{\sum_{k,h}\! \mathbb{E}\!\left[\!Q_{h}^{k}\!\left(\!s_{h}\!,\! a_{h}\!\right)\!-\!c_{h}\!\left(\!s_{h}\!,\! a_{h}\!\right)\!-\!P_{h}\!\left(\!\cdot \!\mid\! s_{h}\!,\! a_{h}\!\right)\! V_{h+1}^{k} \!\mid\! s_{1}\!, \!\pi^*\!,\! P\right]}_{(iv)}.
\end{align*}
$\text{Term}\ (iii)$ is also called the OMD term. Using Lemma \ref{lem:omd_in_po}, we have
$$\text{Term}\ (iii)\leq \tilde{O}(\sqrt{AH^4K'}).$$
We note that the reason why we change the Bregman term and the learning rate schedule is to ensure a $\text{Term}\ (iii)\leq \tilde{O}(\sqrt{K'})$ type bound. In \cite{shani2020optimistic}, they use the KL-divergence as the Bregman term and a fixed learning rate that is a function of $K$, hence their bound is $\text{Term}\ (iii)\leq \tilde{O}(\sqrt{\log(A)H^4K})$. Although our choice of the $\ell_2$ penalty term and decreasing learning rate leads to a larger dependence on $A$, but it has a better dependence on $K'$, which is crucial to ensure that we can learn a good reference function.

Combining with the fact that optimism holds, we have $\text{Term}\ (iv)\leq 0$, therefore the lemma is proved.
\end{proof}
Now we have a $\text{Regret}(K')=\tilde{O}(\sqrt{K'})$ bound for all $K'\leq K$. We show that this immediately implies the following key theorem, which we called the Average Convergence Lemma. The following lemma guarantees that the reference value is a good approximation of $V^*$:
\begin{lemma}[Average convergence of $V^k$] \label{lem:average:convergence}
With the same assumptions and notations in Lemma \ref{lem:estimation_error}, we have for all $K'\leq K, h\leq H$,
\begin{align}\operatorname{Regret}_h(K')\leq \tilde{O} (\sqrt{S^2AH^4K'}). \label{eqn:DRegret}
\end{align}
As a consequence, \propnameacro\;holds, for $\forall K'\leq K,h\leq H$,
\begin{align*}
    \sum_{k=1}^{K'} |V^k_h(s^k_h)-V^*_h(s^k_h)|\leq \tilde{O}(\sqrt{S^2AH^4K'}).
\end{align*}\label{lem:avg}
\end{lemma}
\begin{proof}
For the proof of Equation (\ref{eqn:DRegret}), we use a novel technique called \emph{Forward Induction}, starting from the case when $h=1$,  $\text{Regret}_1(K') = \text{Regret}(K')\leq \tilde{O}(\sqrt{S^2AH^4K'})$, which is true according to Lemma \ref{lem:coarse_regret}. Using induction we can proof Equation (\ref{eqn:DRegret}) for all $h$, which we will leave the details to Appendix \ref{appendix:pf:average:convergence}. For the second statement, we note that 
\begin{align*}
&|V^k_h(s^k_h)-V^*_h(s^k_h)|\\
\leq& |V^k_h(s^k_h)-V^{\pi^k}(s^k_h)|+|V^{\pi^k}_h(s^k_h)-V^*_h(s^k_h)|\\
=&\big(V^{\pi^k}_h(s^k_h)-V^k_h(s^k_h)\big)+\big(V^{\pi^k}_h(s^k_h)-V^*_h(s^k_h)\big).
\end{align*}
These two terms can be handled by Lemma \ref{lem:estimation_error} and Equation (\ref{eqn:DRegret}) separately, hence finished the proof of Lemma \ref{lem:avg}.
\end{proof}

As mentioned before, Theorem \ref{lem:avg} is significant in the sense that it implies the reference value $V^{\mathrm{ref}}_h(s) = \frac{1}{n^k_h(s)}\sum_{i=1}^k V^i_h(s^i_h)\textbf{1}[s^i_h=s]$ is close to the optimal value if $n^k_h(s)$ is large enough, in other words, if $n^k_h(s)\geq C_0\sqrt{k}=\sqrt{S^3AH^3k}$:
\begin{align*}
    &|V^{\mathrm{ref}}_h(s)-V^*_h(s)|\\
    \leq& \frac{1}{n^k_h(s)}\sum_{i=1}^k |V^i_h(s^i_h)-V_h^*(s_h^i)|\textbf{1}[s^i_h=s]\\
    \leq &\frac{1}{n^k_h(s)}\sum_{i=1}^k |V^i_h(s^i_h)-V_h^*(s_h^i)|\\
    \leq & \tilde{O}(\sqrt{\frac{H}{S}})
\end{align*}
This explains the reason why we choose $C_0=\sqrt{S^3AH^3}$ in the algorithm, the $1/\sqrt{S}$ bound serves a key role in reducing the $\sqrt{S}$ factor in the regret. 

Finally, we are ready to prove Theorem \ref{thm:fine_reg}. Combining Lemma \ref{lem:estimation_error} and Equation (\ref{eqn:reg_decomp}),
\begin{align}
    &\sum_{k=1}^K \big(V_1^{\pi^k}(s_1^k)-V_1^*(s_1^k)\big)\notag\\
    \leq & \tilde{O}\big(\sum_{k=1}^K\sum_{h=1}^H b^k_h(s^k_h,a^k_h)\big)+\tilde{O}(\sqrt{AH^4K}).
    \label{eqn:v3}
\end{align}
The only thing left is to prove that 
\begin{align*}
    \sum_{k=1}^K\sum_{h=1}^H b^k_h(s^k_h,a^k_h) \leq \tilde{O}(\sqrt{SAH^3K}+S^{\frac{5}{2}}A^{\frac{5}{4}}H^{\frac{11}{4}}K^{\frac{1}{4}}).
\end{align*} 
By Lemma~\ref{lem:sumbonus}, this is true. Therefore we have finished the proof.

\subsection{How to prove optimism}\label{sec:howto}

In this section, we show how to remove the optimism assumption in Lemma \ref{lem:estimation_error}, Lemma \ref{lem:coarse_regret} and Lemma \ref{lem:average:convergence}. In fact, we will first prove the following lemma , which shows that both optimism and coarse regret bound holds for all episodes.
\begin{lemma}[Coarse analysis of regret]\label{lem:co_wo_opt_2}
For any $K'\in [K]$, we have the following regret bound:
\begin{equation}
\operatorname{Regret}(K') \!= \!\sum_{k=1}^{K'} \big(V^{\pi^k}_1(s_1^k)-V^*_1(s_1^k)\big)\!\leq\! \tilde{O}(\sqrt{S^2AH^4K'}).
\end{equation}
and the optimism holds, namely 
\begin{align}
    -2b_h^k(s, a) \!\leq \!Q_{h}^{k}(s, a)\!-\!c_{h}(s, a)\!-\!P_{h}(\cdot \mid s, a) V_{h+1}^{k}(\cdot)\! \leq \!0.
\end{align}
\end{lemma}
The full proof is in Appendix \ref{app:full}. Here we mainly explain the core idea. To prove Lemma \ref{lem:co_wo_opt_2} we use induction. For the first few episodes of the algorithm, the bonus term is dominated by $\sqrt{\frac{2 \ln \frac{2 S A H T}{\delta^{\prime}}}{n_{h}^{k}(s, a) }}+H\sqrt{\frac{4 S \ln \frac{3 S A H T}{\delta^{\prime}}}{n_{h}^{k}(s, a) }}$, which is the same as in \citet{shani2020optimistic}. Therefore the optimism and coarse regret bound automatically holds.

In the induction step, assume that optimism and coarse regret holds for all $k\leq K^{\prime}$ and $h$. Then by optimism we have that $1$-step regret is bounded by $\tilde{O}(\sqrt{S^2AH^4K^{\prime}})$. Then by Lemma \ref{lem:average:convergence}, we have that $\operatorname{Regret}_h(K')\leq \tilde{O} (\sqrt{S^2AH^4K'})$. This means that if in the $K'+1$-th episode, a new reference function is calculated, then this new reference function must be $\sqrt{H/S}$ accurate, i.e. $|V^{\mathrm{ref}}(s)-V^{*}(s)|\leq O(\sqrt{H/S})$. Hence, the carefully designed bonus guarantees that optimism still holds for $k= K^{\prime}+1$, which finishes the proof. The induction process is shown in Figure \ref{fig:proof}.

Finally, thanks to optimism, we can bound the regret by $\sum_{k=1}^K\sum_{h=1}^H b^k_h(s^k_h,a^k_h)$. The rest follows from the same arguments in the previous section.

\tikzset{global scale/.style={
    scale=#1,
    every node/.append style={scale=#1}
  }
}
\begin{figure}[H]
    \centering
    \begin{tikzpicture}[
roundnode/.style={circle, draw=green!60, fill=green!5, very thick, minimum size=7mm},
squarednode/.style={rectangle, draw=red!60, fill=red!5, very thick, minimum size=5mm}, global scale = 0.8
]
\node[squarednode][align=center]      (optimism)          {Optimism for $\forall (k,h,s,a)\in [K^{\prime}]\times [H]\times\mathcal{S}\times\mathcal{A}$\\ $Q_h^k(s,a)-c_h(s,a)-P_h(\cdot|s,a)V_{h+1}^{k}(\cdot)\leq 0$};
\node[squarednode]        (dynamic)       [below=of optimism] {\emph{1}-st step regret: $\sum_{k=1}^{K^{\prime}}V_1^{\pi^{k}}(s_1^k)-V_1^{*}(s_1^k)\leq O(\sqrt{K^{\prime}})$};
\node[squarednode]      (regret)       [below=of dynamic] {$h$-th step regret:$\sum_{k=1}^{K^{\prime}}V_h^{\pi^{k}}(s_h^k)-V_h^{*}(s_h^k)\leq O(\sqrt{K^{\prime}})$};
\node[squarednode][align=center]         (ref)       [below=of regret] {For any $V^{\mathrm{ref}}(s)$ updated before $K^{\prime}+1$ episode \\$|V^{\mathrm{ref}}(s)-V^{*}(s)|\leq O(\sqrt{H/S})$};
\node[squarednode]      (optimism2)     [below=of ref]                         {Optimism for $\forall (k,h,s,a)\in [K^{\prime}+1]\times [H]\times\mathcal{S}\times\mathcal{A}$\ };

\draw[-{Implies},double distance=3pt] (optimism.south) -- (dynamic.north);
\draw[-{Implies},double distance=3pt] (dynamic.south) -- (regret.north) node[midway,auto,bend right] {Forward Induction};
\draw[-{Implies},double distance=3pt] (regret.south) -- (ref.north);
\draw[-{Implies},double distance=3pt] (ref.south) -- (optimism2.north);
\draw[-{Implies},double distance=3pt] (optimism2.east) -| (5.5,0) -- (optimism.east) node[midway,auto,bend right] {$h\leftarrow h+1$};;
\end{tikzpicture}
    \caption{Proving optimism and $h$-th step regret by induction}
    \label{fig:proof}
\end{figure}
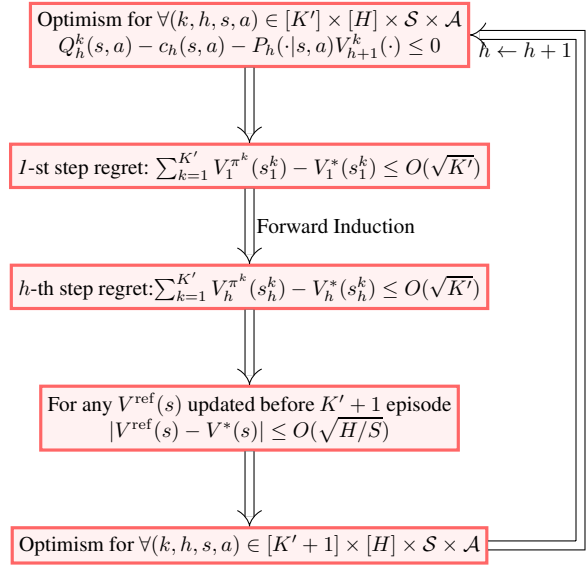

\section{Conclusion and Future Work}

In this paper, we proposed the first optimistic policy optimization algorithm for tabular, episodic RL that can achieve regret guarantee $\tilde{O}(\sqrt{SAH^3K}+\sqrt{AH^4K}+\operatorname{poly}(S,A,H)K^{1/4})$. This algorithm improves upon previous results \citep{shani2020optimistic} and matches the information theoretic limit $\Omega(\sqrt{SAH^3K})$ when $S > H$. Our results also raise a number of promising directions for future work. Theoretically, can we design better policy-based methods that can eliminate the constraint $S > H$? Practically, can we leverage the insight of \algnameacro\,to improve practical policy-based RL algorithms? Specifically, can we design different regularization terms to stabilize the $V/Q$ estimation process to make the algorithm more sample efficient? We look forward to answering these questions in the future.

\section*{Acknowledgement}
Liwei Wang was supported by Exploratory Research Project
of Zhejiang Lab (No. 2022RC0AN02),
Project 2020BD006 supported by PKUBaidu Fund, the major
key project of PCL (PCL2021A12).

\nocite{langley00}

\bibliography{example_paper}

\begin{thebibliography}{36}
\providecommand{\natexlab}[1]{#1}
\providecommand{\url}[1]{\texttt{#1}}
\expandafter\ifx\csname urlstyle\endcsname\relax
  \providecommand{\doi}[1]{doi: #1}\else
  \providecommand{\doi}{doi: \begingroup \urlstyle{rm}\Url}\fi

\bibitem[Agarwal et~al.(2021)Agarwal, Kakade, Lee, and
  Mahajan]{agarwal2021theory}
Agarwal, A., Kakade, S.~M., Lee, J.~D., and Mahajan, G.
\newblock On the theory of policy gradient methods: Optimality, approximation,
  and distribution shift.
\newblock \emph{Journal of Machine Learning Research}, 22\penalty0
  (98):\penalty0 1--76, 2021.

\bibitem[Azar et~al.(2017)Azar, Osband, and Munos]{azar2017minimax}
Azar, M.~G., Osband, I., and Munos, R.
\newblock Minimax regret bounds for reinforcement learning.
\newblock In \emph{International Conference on Machine Learning}, pp.\
  263--272. PMLR, 2017.

\bibitem[Beck \& Teboulle(2003)Beck and Teboulle]{beck2003mirror}
Beck, A. and Teboulle, M.
\newblock Mirror descent and nonlinear projected subgradient methods for convex
  optimization.
\newblock \emph{Operations Research Letters}, 31\penalty0 (3):\penalty0
  167--175, 2003.

\bibitem[Bhandari \& Russo(2019)Bhandari and Russo]{bhandari2019global}
Bhandari, J. and Russo, D.
\newblock Global optimality guarantees for policy gradient methods.
\newblock \emph{arXiv preprint arXiv:1906.01786}, 2019.

\bibitem[Cai et~al.(2020)Cai, Yang, Jin, and Wang]{cai2020provably}
Cai, Q., Yang, Z., Jin, C., and Wang, Z.
\newblock Provably efficient exploration in policy optimization.
\newblock In \emph{International Conference on Machine Learning}, pp.\
  1283--1294. PMLR, 2020.

\bibitem[Domingues et~al.(2021)Domingues, M{\'e}nard, Kaufmann, and
  Valko]{domingues2021episodic}
Domingues, O.~D., M{\'e}nard, P., Kaufmann, E., and Valko, M.
\newblock Episodic reinforcement learning in finite mdps: Minimax lower bounds
  revisited.
\newblock In \emph{Algorithmic Learning Theory}, pp.\  578--598. PMLR, 2021.

\bibitem[Fazel et~al.(2018)Fazel, Ge, Kakade, and Mesbahi]{fazel2018global}
Fazel, M., Ge, R., Kakade, S., and Mesbahi, M.
\newblock Global convergence of policy gradient methods for the linear
  quadratic regulator.
\newblock In \emph{International Conference on Machine Learning}, pp.\
  1467--1476. PMLR, 2018.

\bibitem[Fei et~al.(2020)Fei, Yang, Wang, and Xie]{fei2020dynamic}
Fei, Y., Yang, Z., Wang, Z., and Xie, Q.
\newblock Dynamic regret of policy optimization in non-stationary environments.
\newblock \emph{arXiv preprint arXiv:2007.00148}, 2020.

\bibitem[Gu et~al.(2017)Gu, Holly, Lillicrap, and Levine]{gu2017deep}
Gu, S., Holly, E., Lillicrap, T., and Levine, S.
\newblock Deep reinforcement learning for robotic manipulation with
  asynchronous off-policy updates.
\newblock In \emph{2017 IEEE international conference on robotics and
  automation (ICRA)}, pp.\  3389--3396. IEEE, 2017.

\bibitem[Haarnoja et~al.(2018)Haarnoja, Zhou, Hartikainen, Tucker, Ha, Tan,
  Kumar, Zhu, Gupta, Abbeel, et~al.]{haarnoja2018soft}
Haarnoja, T., Zhou, A., Hartikainen, K., Tucker, G., Ha, S., Tan, J., Kumar,
  V., Zhu, H., Gupta, A., Abbeel, P., et~al.
\newblock Soft actor-critic algorithms and applications.
\newblock \emph{arXiv preprint arXiv:1812.05905}, 2018.

\bibitem[He et~al.(2021)He, Zhou, and Gu]{he2021nearly}
He, J., Zhou, D., and Gu, Q.
\newblock Nearly optimal regret for learning adversarial mdps with linear
  function approximation.
\newblock \emph{arXiv preprint arXiv:2102.08940}, 2021.

\bibitem[Jaksch et~al.(2010)Jaksch, Ortner, and Auer]{jaksch2010near}
Jaksch, T., Ortner, R., and Auer, P.
\newblock Near-optimal regret bounds for reinforcement learning.
\newblock \emph{Journal of Machine Learning Research}, 11\penalty0 (4), 2010.

\bibitem[Jin et~al.(2018)Jin, Allen-Zhu, Bubeck, and Jordan]{jin2018q}
Jin, C., Allen-Zhu, Z., Bubeck, S., and Jordan, M.~I.
\newblock Is q-learning provably efficient?
\newblock \emph{arXiv preprint arXiv:1807.03765}, 2018.

\bibitem[Kakade(2001)]{kakade2001natural}
Kakade, S.~M.
\newblock A natural policy gradient.
\newblock \emph{Advances in neural information processing systems}, 14, 2001.

\bibitem[Kalashnikov et~al.(2018)Kalashnikov, Irpan, Pastor, Ibarz, Herzog,
  Jang, Quillen, Holly, Kalakrishnan, Vanhoucke, et~al.]{kalashnikov2018qt}
Kalashnikov, D., Irpan, A., Pastor, P., Ibarz, J., Herzog, A., Jang, E.,
  Quillen, D., Holly, E., Kalakrishnan, M., Vanhoucke, V., et~al.
\newblock Qt-opt: Scalable deep reinforcement learning for vision-based robotic
  manipulation.
\newblock \emph{arXiv preprint arXiv:1806.10293}, 2018.

\bibitem[Lancewicki et~al.(2020)Lancewicki, Rosenberg, and
  Mansour]{lancewicki2020learning}
Lancewicki, T., Rosenberg, A., and Mansour, Y.
\newblock Learning adversarial markov decision processes with delayed feedback.
\newblock \emph{arXiv preprint arXiv:2012.14843}, 2020.

\bibitem[Levine et~al.(2016)Levine, Finn, Darrell, and Abbeel]{levine2016end}
Levine, S., Finn, C., Darrell, T., and Abbeel, P.
\newblock End-to-end training of deep visuomotor policies.
\newblock \emph{The Journal of Machine Learning Research}, 17\penalty0
  (1):\penalty0 1334--1373, 2016.

\bibitem[Liu et~al.(2019)Liu, Cai, Yang, and Wang]{liu2019neural}
Liu, B., Cai, Q., Yang, Z., and Wang, Z.
\newblock Neural trust region/proximal policy optimization attains globally
  optimal policy.
\newblock In \emph{Advances in Neural Information Processing Systems}, pp.\
  10565--10576, 2019.

\bibitem[Luo et~al.(2021)Luo, Wei, and Lee]{luo2021policy}
Luo, H., Wei, C.-Y., and Lee, C.-W.
\newblock Policy optimization in adversarial mdps: Improved exploration via
  dilated bonuses.
\newblock \emph{arXiv preprint arXiv:2107.08346}, 2021.

\bibitem[Maurer \& Pontil(2009)Maurer and Pontil]{maurer2009empirical}
Maurer, A. and Pontil, M.
\newblock Empirical bernstein bounds and sample variance penalization.
\newblock \emph{arXiv preprint arXiv:0907.3740}, 2009.

\bibitem[Menard et~al.(2021)Menard, Domingues, Shang, and Valko]{menard2021ucb}
Menard, P., Domingues, O.~D., Shang, X., and Valko, M.
\newblock Ucb momentum q-learning: Correcting the bias without forgetting.
\newblock \emph{arXiv preprint arXiv:2103.01312}, 2021.

\bibitem[Orabona(2019)]{orabona2019modern}
Orabona, F.
\newblock A modern introduction to online learning.
\newblock \emph{arXiv preprint arXiv:1912.13213}, 2019.

\bibitem[Schulman et~al.(2015)Schulman, Levine, Abbeel, Jordan, and
  Moritz]{schulman2015trust}
Schulman, J., Levine, S., Abbeel, P., Jordan, M., and Moritz, P.
\newblock Trust region policy optimization.
\newblock In \emph{International conference on machine learning}, pp.\
  1889--1897. PMLR, 2015.

\bibitem[Schulman et~al.(2017)Schulman, Wolski, Dhariwal, Radford, and
  Klimov]{schulman2017proximal}
Schulman, J., Wolski, F., Dhariwal, P., Radford, A., and Klimov, O.
\newblock Proximal policy optimization algorithms.
\newblock \emph{arXiv preprint arXiv:1707.06347}, 2017.

\bibitem[Shani et~al.(2020)Shani, Efroni, Rosenberg, and
  Mannor]{shani2020optimistic}
Shani, L., Efroni, Y., Rosenberg, A., and Mannor, S.
\newblock Optimistic policy optimization with bandit feedback.
\newblock In \emph{International Conference on Machine Learning}, pp.\
  8604--8613. PMLR, 2020.

\bibitem[Silver et~al.(2016)Silver, Huang, Maddison, Guez, Sifre, Van
  Den~Driessche, Schrittwieser, Antonoglou, Panneershelvam, Lanctot,
  et~al.]{silver2016mastering}
Silver, D., Huang, A., Maddison, C.~J., Guez, A., Sifre, L., Van Den~Driessche,
  G., Schrittwieser, J., Antonoglou, I., Panneershelvam, V., Lanctot, M.,
  et~al.
\newblock Mastering the game of go with deep neural networks and tree search.
\newblock \emph{nature}, 529\penalty0 (7587):\penalty0 484--489, 2016.

\bibitem[Silver et~al.(2017)Silver, Schrittwieser, Simonyan, Antonoglou, Huang,
  Guez, Hubert, Baker, Lai, Bolton, et~al.]{silver2017mastering}
Silver, D., Schrittwieser, J., Simonyan, K., Antonoglou, I., Huang, A., Guez,
  A., Hubert, T., Baker, L., Lai, M., Bolton, A., et~al.
\newblock Mastering the game of go without human knowledge.
\newblock \emph{nature}, 550\penalty0 (7676):\penalty0 354--359, 2017.

\bibitem[Sutton et~al.(1999)Sutton, McAllester, Singh, Mansour,
  et~al.]{sutton1999policy}
Sutton, R.~S., McAllester, D.~A., Singh, S.~P., Mansour, Y., et~al.
\newblock Policy gradient methods for reinforcement learning with function
  approximation.
\newblock In \emph{NIPs}, volume~99, pp.\  1057--1063. Citeseer, 1999.

\bibitem[Wang et~al.(2019)Wang, Cai, Yang, and Wang]{wang2019neural}
Wang, L., Cai, Q., Yang, Z., and Wang, Z.
\newblock Neural policy gradient methods: Global optimality and rates of
  convergence.
\newblock \emph{arXiv preprint arXiv:1909.01150}, 2019.

\bibitem[Weissman et~al.(2003)Weissman, Ordentlich, Seroussi, Verdu, and
  Weinberger]{weissman2003inequalities}
Weissman, T., Ordentlich, E., Seroussi, G., Verdu, S., and Weinberger, M.~J.
\newblock Inequalities for the l1 deviation of the empirical distribution.
\newblock \emph{Hewlett-Packard Labs, Tech. Rep}, 2003.

\bibitem[Zanette \& Brunskill(2019)Zanette and Brunskill]{zanette2019tighter}
Zanette, A. and Brunskill, E.
\newblock Tighter problem-dependent regret bounds in reinforcement learning
  without domain knowledge using value function bounds.
\newblock In \emph{International Conference on Machine Learning}, pp.\
  7304--7312. PMLR, 2019.

\bibitem[Zanette et~al.(2020)Zanette, Lazaric, Kochenderfer, and
  Brunskill]{zanette2020learning}
Zanette, A., Lazaric, A., Kochenderfer, M., and Brunskill, E.
\newblock Learning near optimal policies with low inherent bellman error.
\newblock In \emph{International Conference on Machine Learning}, pp.\
  10978--10989. PMLR, 2020.

\bibitem[Zanette et~al.(2021)Zanette, Cheng, and
  Agarwal]{zanette2021cautiously}
Zanette, A., Cheng, C.-A., and Agarwal, A.
\newblock Cautiously optimistic policy optimization and exploration with linear
  function approximation.
\newblock \emph{arXiv preprint arXiv:2103.12923}, 2021.

\bibitem[Zhang et~al.(2020{\natexlab{a}})Zhang, Ji, and
  Du]{zhang2020reinforcement}
Zhang, Z., Ji, X., and Du, S.~S.
\newblock Is reinforcement learning more difficult than bandits? a near-optimal
  algorithm escaping the curse of horizon.
\newblock \emph{arXiv preprint arXiv:2009.13503}, 2020{\natexlab{a}}.

\bibitem[Zhang et~al.(2020{\natexlab{b}})Zhang, Zhou, and Ji]{zhang2020almost}
Zhang, Z., Zhou, Y., and Ji, X.
\newblock Almost optimal model-free reinforcement learningvia
  reference-advantage decomposition.
\newblock \emph{Advances in Neural Information Processing Systems}, 33,
  2020{\natexlab{b}}.

\bibitem[Zhong et~al.(2021)Zhong, Yang, Wang, and
  Szepesv{\'a}ri]{zhong2021optimistic}
Zhong, H., Yang, Z., Wang, Z., and Szepesv{\'a}ri, C.
\newblock Optimistic policy optimization is provably efficient in
  non-stationary mdps.
\newblock \emph{arXiv preprint arXiv:2110.08984}, 2021.

\end{thebibliography}
\bibliographystyle{icml2022}

\newpage
\appendix
\onecolumn
\section{Regret Decomposition and Failure Events}

 \subsection{Regret Decomposition}

 For any $(k, h) \in [K] \times [H]$, we define 
 \begin{equation}
     \begin{aligned} \label{eq:def:martingale}
     \zeta_{k, h}^1 &= [V_{h}^k(s_h^k) - V_{h}^{\pi^k}(s_h^k)] - [Q_{h}^k(s_h^k, a_h^k) - Q_{h}^{\pi^k}(s_h^k, a_h^k)], \\
     \zeta_{k, h}^2 &= [P_h(\cdot \mid s_h^k, a_h^k)V_{h+1}^k(\cdot) - P_h(\cdot \mid s_h^k, a_h^k)V_{h+1}^{\pi^k}(\cdot)] - [V_{h+1}^k(s_{h+1}^k) - V_{h+1}^{\pi^k}(s_{h+1}^k)].
     \end{aligned}
     \end{equation}
     By the definition, we have that $\zeta_{k, h}^1$ and $\zeta_{k, h}^2$ represent the randomness of executing a stochastic policy $\pi_h^k(\cdot \mid s_h^k)$ and the randomness of observing the next state from stochastic transition kernel $P_h(\cdot \mid s_h^k, a_h^k)$, respectively.
 With these notations, we have the following standard regret decomposition lemma \citep{cai2020provably,shani2020optimistic}.
 \begin{lemma}[Regret Decomposition] \label{lemma:regret:decomposition}
     For any $(K', h') \in [K] \times [H]$, it holds that 
     \begin{align*}
         \operatorname{Regret}_{h'}(K') & = \underbrace{\sum_{k=1}^{K'} \big(V_{h'}^{\pi^{k}}(s_{h'}^k)-V_{h'}^{k}(s_{h'}^k)\big)}_{{(i)}} + \underbrace{\sum_{k=1}^{K'} \big(V_{h'}^{k}(s_{h'}^k)-V_{h'}^{*}(s_{h'}^k)\big)}_{{(ii)}}\\
         &=\underbrace{\sum_{k=1}^{K'} \sum_{h=h'}^{H} \left[c_{h}(s_{h}^k, a_{h}^k) + P_{h}(\cdot \mid s_{h}^k, a_{h}^k) V_{h+1}^{k}(\cdot) - Q_{h}^{k}(s_{h}^k, a_{h}^k) \right]}_{{(i.1)}} + \underbrace{\sum_{k=1}^{K'} \sum_{h=h'}^{H} (\zeta_{k, h}^1 + \zeta_{k, h}^2)}_{{(i.2)}}\\
         &+\underbrace{\sum_{k=1}^{K'} \sum_{h=h'}^{H} \mathbb{E}\Big[\langle Q_{h}^{k}(s_{h}, \cdot), \pi_{h}^{k}(\cdot \mid s_{h})-\pi_{h}(\cdot \mid s_{h})\rangle \mid s_{h'} = s_{h'}^k, \pi, P\Big]}_{(ii.1)}\\
         &+\underbrace{\sum_{k=1}^{K'} \sum_{h=h'}^{H} \mathbb{E}\left[Q_{h}^{k}\left(s_{h}, a_{h}\right)-c_{h}\left(s_{h}, a_{h}\right)-P_{h}\left(\cdot \mid s_{h}, a_{h}\right) V_{h+1}^{k}(\cdot) \mid s_{h'}=s_{h'}^k, \pi, P\right]}_{(ii.2)} 
     \end{align*}

     \begin{proof}
         See \citet{cai2020provably} for a detailed proof.
     \end{proof}
 \end{lemma}

\subsection{Failure Events}
\begin{definition} \label{def:failure:event}
We define the following failure events:
\begin{gather*}
F_k^0 = \left\{ \exists h: \left| \sum_{k'=1}^{k} \sum_{h'=h}^{H} (\zeta_{k', h'}^1 + \zeta_{k', h'}^2)\right| \ge \sqrt{16H^3K\ln\frac{2H}{\delta'}}\right\}, \\
F_{k}^{1}=\left\{\exists s, a, h:\left|c_{h}(s, a)-\bar{c}_{h}^{k}(s, a)\right| \geq \sqrt{\frac{2 \ln \frac{2 S A H T}{\delta^{\prime}}}{n_{h}^{k}(s, a) }}\right\},\\
F_{k}^{2}=\left\{\exists s, a, h:\left\|P_{h}(\cdot \mid s, a)-\bar{P}_{h}^{k}(\cdot \mid s, a)\right\|_{1} \geq \sqrt{\frac{4 S \ln \frac{3 S A H T}{\delta^{\prime}}}{n_{h}^{k}(s, a) }}\right\},\\
F_{k}^{3}=\left\{\exists s, a, h:\left|\left(P_{h}(\cdot \mid s, a)-\bar{P}_{h}^{k}(\cdot \mid s, a)\right)V_{h}^{*}(\cdot)\right| \geq H\sqrt{\frac{4  \ln \frac{3 S A H T}{\delta^{\prime}}}{n_{h}^{k}(s, a) }}+\frac{2H\ln\frac{2SAHT}{\delta^{\prime}}}{3n_{h}^{k}(s, a) }\right\},\\
F_{k}^{4}=\left\{\exists s^{\prime}, s, a, h: \left|\bar{P}_{h}^{k}(s^{\prime} \mid s, a)-P_{h}(s^{\prime} \mid s, a)\right| \geq \sqrt{\frac{2\bar{P}_{h}^k(s^{\prime} \mid s, a)(1-\bar{P}_{h}^k(s^{\prime} \mid s, a)) \ln \left(\frac{2 SAHK}{\delta^{\prime}}\right)}{n_{h}^{k}(s, a)-1}}+\frac{7 \ln \left(\frac{2 SAHK}{\delta^{\prime}}\right)}{3 n_{h}^{k}(s, a)}\right\},\\
F= \bigcup_{k=1}^{K} \left( F^{0}_k \cup F_{k}^{1}\cup F_{k}^{2}\cup F_{k}^{3} \cup F_{k}^{4}\right),
\end{gather*}
where $\delta^{\prime}=\frac{\delta}{5}$.
\end{definition}
\begin{lemma}
It holds that 
\begin{align*}
    \Pr(F) \le \delta. 
\end{align*}
\end{lemma}
\begin{proof}
These standard concentration inequalities also appear in \citet{azar2017minimax,cai2020provably,shani2020optimistic}. For completeness, we present the proof sketch here.

Note that the martingales $\zeta_{k, h}^1$ and $\zeta_{k, h}^2$ defined in \eqref{eq:def:martingale} satisfy $|\zeta_{k, h}^1 + \zeta_{k, h}^2| \le 4H$. By Azuma-Hoeffding inequality, we have $\Pr(F^0) \le \delta'$.

Let $F^{1}=\bigcup_{k=1}^{K} F_{k}^{1} .$  By Hoeffding's inequality, we have 
\begin{equation}
\operatorname{Pr}\left\{\left|c_{h}(s, a)-\bar{c}_{h}^{k}(s, a)\right| \geq \sqrt{\frac{2 \ln \frac{1}{\delta^{\prime}}}{n_{h}^{k}(s, a) }}\right\}\leq \delta^{\prime}
\end{equation}
Using a union bound over all $s,a$ and all possible values of $n_{k}(s, a)$ and $k$., we have $\operatorname{Pr}\left\{F^{c}\right\} \leq \delta^{\prime}$.

Let $F^{2}=\bigcup_{k=1}^{K} F_{k}^{2}$. Then $\operatorname{Pr}\left\{F^{2}\right\} \leq \delta^{\prime}$, which is implied by \cite{weissman2003inequalities} while applying union bound on all $s,a$ and all possible values of $n_{k}(s, a)$ and $k$.

Let $F^{3}=\bigcup_{k=1}^{K} F_{k}^{3}$. According to \cite{azar2017minimax}, we have that with probability at least $1-\delta^{\prime}$
\begin{equation}
\left|\left[\left(P_{h}-\bar{P}_{h}^k\right) V_{h}^{*}\right](s, a)\right| \leq \sqrt{\frac{2 H^2 \ln \left(\frac{2 HK}{\delta^{\prime}}\right)}{n_{h}^{k}(s, a)}}+\frac{2 H \ln \left(\frac{2 HK}{\delta^{\prime}}\right)}{3 n_{h}^{k}(s, a)}.
\end{equation}
Take a union bound over $s,a,k$, we have $\operatorname{Pr}\left\{F^{3}\right\} \leq \delta^{\prime}$.

Let $F^{4}=\bigcup_{k=1}^{K} F_{k}^{4}$. The Empirical Bernstein inequality (Theorem 4 in \cite{maurer2009empirical}) combined with a union bound argument on $s,a,s^{\prime}$,$n_{h}^{k}(s,a)$ also implies the following bound holds with probability at least $1-\delta^{\prime}$:
\begin{equation}
\left|\bar{P}_h^{k}(s^{\prime} \mid s, a)-P_h(s^{\prime} \mid s, a)\right| \leq \sqrt{\frac{2\bar{P}(s^{\prime} \mid s, a)(1-\bar{P}(s^{\prime} \mid s, a)) \ln \left(\frac{2 SAHK}{\delta^{\prime}}\right)}{n_{h}^{k}(s, a)-1}}+\frac{7 \ln \left(\frac{2 SAHK}{\delta^{\prime}}\right)}{3 n_{h}^{k}(s, a)}
\end{equation}
Therefore, $\operatorname{Pr}\{F^{4}\}\leq \delta^{\prime}$

Finally, take a union bound with $\delta^{\prime}=\frac{\delta}{5}$, we have $\operatorname{Pr}\left\{F\right\} \leq \delta$.
\end{proof}

Below we will assume that the failure event $F$ does not happen, which is with high probability.

\section{Missing Proofs for Section \ref{sec:proof:sketch}} \label{appendix:full:proof}

\subsection{Proof of Lemma \ref{lem:regret:bound}} \label{appendix:pf:lem:regret:bound}
\begin{proof}
By Lemma \ref{lemma:regret:decomposition}, we have 
    \begin{align*}
        &\sum_{k=1}^{K'}\big(V^{\pi_k}_h(s^k_h)-V^k_{h}(s^k_h)\big) \\
         +&\sum_{k=1}^{K'} \sum_{h=h'}^{H} \mathbb{E}\Big[\langle Q_{h}^{k}(s_{h}, \cdot), \pi_{h}^{k}(\cdot \mid s_{h})-\pi_{h}(\cdot \mid s_{h})\rangle \mid s_{h'} = s_{h'}^k, \pi, P\Big] + \sum_{k=1}^{K'} \sum_{h=h'}^{H} (\zeta_{k, h}^1 + \zeta_{k, h}^2).
    \end{align*}
    Here $\zeta_{k, h}^1$ and $\zeta_{k, h}^2$ are martingales defined in \eqref{eq:def:martingale} satisfying $|\zeta_{k, h}^1 + \zeta_{k, h}^2| \le 4H$.
    Under Azuma-Hoeffding inequalities and the optimistic assumption that $-2 b^k_h(s,a)\leq Q^k_h(s,a)-c_h(s,a)-P_h(\cdot|s,a)V^k_{h+1}\leq 0$, we finish the proof of Lemma \ref{lem:regret:bound}.
\end{proof}

\subsection{Full Proof of Theorem \ref{thm:fine_reg}} \label{app:full}
As discussed in Section \ref{sec:proof:sketch}, we only need to prove the coarse regret in Lemma \ref{lem:co_wo_opt_2}, but this time without the optimism assumption. We restate the lemma for ease of reading.
\begin{lemma}[Coarse analysis of dynamic regret, restatement of Lemma \ref{lem:co_wo_opt_2}]\label{lem:co_wo_opt}
Conditioned on $F^c$, for any $K'\in [K]$, we have the following regret bound:
\begin{equation}\label{eqn:coarse_reg}
\operatorname{Regret}(K') = \sum_{k=1}^{K'} \big(V^{\pi^k}_1(s_1^k)-V^*_1(s_1^k)\big)\leq \tilde{O}(\sqrt{S^2AH^4K'}).
\end{equation}
and the optimism holds, namely 
\begin{align}\label{eqn:opt_coarse}
    -2b_h^k(s, a) \le Q_{h}^{k}(s, a)-c_{h}(s, a)-P_{h}(\cdot \mid s, a) V_{h+1}^{k}(\cdot) \leq 0.
\end{align}
\end{lemma}
\begin{remark}
In our algorithm, we change the Bregman penalty term from KL-divergence $d_{KL}(\pi_h\|\pi^k_h)$ to $\|\pi_h-\pi^k_h\|_2^2$, we also let the learning rate to be of the form $\eta_t = O(\frac{1}{\sqrt{t}})$ instead of a constant depending on $K$. We note that although this will cause more regret in the OMD term, it is crucial to obtain a $\text{Regret}(K')=O(\sqrt{K'})$ bound instead of $\text{Regret}(K')=O(K')$.
\end{remark}
\begin{proof}
First, by Lemma \ref{lemma:regret:decomposition}, we decompose the regret in the following way.
\begin{align*}
    &\sum_{k=1}^{K^{\prime}} \big(V_{1}^{\pi^{k}}\left(s_{1}\right)-V_{1}^{*}\left(s_{1}\right)\big)\\
    =&\underbrace{\sum_{k=1}^{K^{\prime}} \sum_{h=1}^{H} \left[c_{h}(s_{h}^k, a_{h}^k) + P_{h}(\cdot \mid s_{h}^k, a_{h}^k) V_{h+1}^{k}(\cdot) - Q_{h}^{k}(s_{h}^k, a_{h}^k) \right]}_{{(i)}} + \underbrace{\sum_{k=1}^{K} \sum_{h=1}^{H} (\zeta_{k, h}^1 + \zeta_{k, h}^2)}_{{(ii)}}\\
         &+\underbrace{\sum_{k=1}^{K^{\prime}} \sum_{h=1}^{H} \mathbb{E}\Big[\langle Q_{h}^{k}(s_{h}, \cdot), \pi_{h}^{k}(\cdot \mid s_{h})-\pi_{h}(\cdot \mid s_{h})\rangle \mid s_{1} = s_{1}^k, \pi, P\Big]}_{(iii)}\\
         &+\underbrace{\sum_{k=1}^{K^{\prime}} \sum_{h=1}^{H} \mathbb{E}\left[Q_{h}^{k}\left(s_{h}, a_{h}\right)-c_{h}\left(s_{h}, a_{h}\right)-P_{h}\left(\cdot \mid s_{h}, a_{h}\right) V_{h+1}^{k}(\cdot) \mid s_{1}=s_{1}^k, \pi, P\right]}_{(iv)} .
\end{align*}

We first prove \eqref{eqn:coarse_reg} and \eqref{eqn:opt_coarse} for early episodes, namely when $n_{h}^{k}<C_0 \sqrt{k}$ for all $s$ (which is true at the beginning of the algorithm), then the regret bound holds. Note that in this case, conditioned on $F^c$, we have $\frac{SH\sqrt{C_0}{K^{\prime}}^{1/4}}{n^k_h(s,a)}\geq \sqrt{\frac{H^2S}{n^k_h(s,a)}}$, which implies $b_h^k(s,a)=\sqrt{\frac{2 \ln \frac{2 S A H T}{\delta^{\prime}}}{n_{h}^{k}(s, a)  }}+H\sqrt{\frac{4 S \ln \frac{3 S A H T}{\delta^{\prime}}}{n_{h}^{k}(s, a)  }}$. In this case, by Lemma \ref{lem:opt_beg}, we have $-2b_h^k(s, a) \le Q_{h}^{k}(s, a)-c_{h}(s, a)-P_{h}(\cdot \mid s, a) V_{h+1}^{k}(\cdot) \leq 0$, which further implies that 
\begin{align*}
    \text{Term }  (i) \leq & O\big(\sum_{k=1}^{K^{\prime}}\sum_{h=1}^H b^k_h(s^k_h,a^k_h)\big)\\
    \leq & O\big(\sum_{k=1}^{K^{\prime}}\sum_{h=1}^H \sqrt{\frac{2 \ln \frac{2 S A H T}{\delta^{\prime}}}{n_{h}^{k}(s, a)  }}+H\sqrt{\frac{4 S \ln \frac{3 S A H T}{\delta^{\prime}}}{n_{h}^{k}(s, a)  }} \big)\\
    \leq &  \tilde{O}(\sqrt{S^2AH^4K^{\prime}}),
\end{align*}
and
\begin{equation*}
    \text{Term } (iv) \leq 0.
\end{equation*}
Also, by Lemma \ref{lem:opt_beg}, we know that \eqref{eqn:opt_coarse} holds in this period.
Applying Lemma \ref{lem:omd_in_po} to Term $(iii)$, we have
\begin{equation*}
    \text{Term } (iii) \leq \tilde{O}(\sqrt{AH^4{K^{\prime}}})
\end{equation*}
Therefore, under event $F^c$, we have
\begin{equation*}
    \sum_{k=1}^{K^{\prime}} \big(V_{1}^{\pi^{k}}\left(s_{1}\right)-V_{1}^{*}\left(s_{1}\right)\big) \leq \tilde{O}(\sqrt{S^2AH^4{K^{\prime}}})
\end{equation*}

Next, we prove \eqref{eqn:coarse_reg} and \eqref{eqn:opt_coarse} for the remaining episodes. We prove this claim by induction. In fact, we will prove the following claim: for each episode, \eqref{eqn:coarse_reg} and \eqref{eqn:opt_coarse} hold.
We have shown that this claim holds for the first episodes.

Assume that for $k= K^{\prime}-1$, we have $\sum_{k=1}^{K^{\prime}-1} (V_{1}^{\pi^{k}}\left(s_{1}\right)-V_{1}^{*}\left(s_{1}\right))\leq \tilde{O}(\sqrt{S^2AH^4(K^{\prime}-1)})$, we want to prove that for $k= K^{\prime}$, we have $\sum_{k=1}^{K^{\prime}} (V_{1}^{\pi^{k}}\left(s_{1}\right)-V_{1}^{*}\left(s_{1}\right))\leq \tilde{O}(\sqrt{S^2AH^4K^{\prime}})$.

If there exists $(h,s,a)$ such that $n_h^{K^{\prime}}(s)\geq C_0\sqrt{K^{\prime}}$ (i.e. we construct $V_{h}^{\mathrm{ref}}(s)$ in this episode), then since \eqref{eqn:coarse_reg} and \eqref{eqn:opt_coarse} holds for previous episodes, using Lemma \ref{lem:average:convergence} we have 
\begin{align*}
    \left|V_h^{\mathrm{ref}}(s)-V_{h}^{*}(s)\right|&=\frac{1}{n^{K^{\prime}}_h(s)}\left|\sum_{i=1}^{K^{\prime}} \left[V^i_h(s^i_h)-V_{h}^{*}(s)\right]\textbf{1}[s^i_h=s]\right|\\
    &\leq \frac{1}{n^{K^{\prime}}_h(s)}\sum_{i=1}^{K^{\prime}} \left|V^i_h(s^i_h)-V_{h}^{*}(s)\right|\textbf{1}[s^i_h=s]\\
    &\leq \frac{1}{C_0\sqrt{K^{\prime}}}\sum_{i=1}^{K^{\prime}} \left|V^i_h(s)-V_{h}^{*}(s)\right|\\
    &\leq \frac{1}{C_0\sqrt{K^{\prime}}}\sqrt{S^2AH^4K^{\prime}}=\frac{\sqrt{S^2AH^4}}{C_0}
\end{align*}

Therefore, by Lemma \ref{lem:coarseopt} we know that $Q_{h}^{k}$ is an optimistic estimation of $Q_{h}^{*}$. Together with Lemmas \ref{lemma:regret:decomposition} and \ref{lem:omd_in_po}, we have
\begin{align*}
    &\sum_{k=1}^{K^{\prime}} \big(V_{1}^{\pi^{k}}\left(s_{1}\right)-V_{1}^{*}\left(s_{1}\right) \big)\\
    \leq & O\big(\sum_{k=1}^{K^{\prime}}\sum_{h=1}^H b^k_h(s^k_h,a^k_h)\big)+O(\sqrt{AH^4{K^{\prime}}})\\
    \leq & O\big(\sum_{k=1}^{K^{\prime}}\sum_{h=1}^H \sqrt{\frac{2 \ln \frac{2 S A H T}{\delta^{\prime}}}{n_{h}^{k}(s, a)  }}+H\sqrt{\frac{4 S \ln \frac{3 S A H T}{\delta^{\prime}}}{n_{h}^{k}(s, a)  }} \big)+O(\sqrt{AH^4{K^{\prime}}})\\
    \leq &  \tilde{O}(\sqrt{S^2AH^4{K^{\prime}}}),
\end{align*}
which concludes our proof.

\end{proof}

\subsection{Proof of Lemma \ref{lem:average:convergence}} \label{appendix:pf:average:convergence}
\begin{proof}
        We have
        \begin{align*}
            &\sum_{k=1}^{K'} |V_h^k(s^k_h)-V^*_h(s^k_h)|\\
            \leq& \sum_{k=1}^{K'} |V_h^k(s^k_h)-V^{\pi^k}_h(s^k_h)|+\sum_{k=1}^{K'} |V_h^{\pi^k}(s^k_h)-V^*_h(s^k_h)|\\
            =&\sum_{k=1}^{K'} (V^{\pi^k}_h(s^k_h)-V_h^k(s^k_h))+\sum_{k=1}^{K'} (V_h^{\pi^k}(s^k_h)-V^*_h(s^k_h))
        \end{align*}
        where in the last step we use optimism and definition of $V^{*}$. By Lemma \ref{lem:regret:bound}, the first term $\sum_{k=1}^{K'} (V^{\pi^k}_h(s^k_h)-V_h^k(s^k_h))$ can be bounded as
        \begin{equation}
        \begin{aligned}
        &\sum_{k=1}^{K'} \bigl(V_{h}^{\pi_{k}}\left(s_{h}\right)-V_{h}^{k}\left(s_{h}\right)\bigr) \\
        \leq & O\Big(\sum_{k=1}^{K'} \sum_{h^{\prime}=h}^{H} b_{h^{\prime}}^{k}(s_{h^{\prime}}^k,a_{h^{\prime}}^k)  \Big) + O(\sqrt{H^3K'})\\
        \leq & O\Big(\sum_{k=1}^{K'} \sum_{h^{\prime}=h}^{H} \sqrt{\frac{H^2S}{n_{h^{\prime}}^{k}(s_{h^{\prime}}^k,a_{h^{\prime}}^k)}}\Big) + O(\sqrt{H^3K'})\\ 
        \leq & \Tilde{O}(\sqrt{S^2AH^4K'}).
        \end{aligned}
        \end{equation}
        For the second term $\sum^{K'}_{k=1} (V_h^{\pi^k}(s^k_h)-V^*_h(s^k_h))$, we use a forward induction trick:
        
        First we have $\sum^{K'}_{k=1} (V_h^{\pi^k}(s^k_h)-V^*_h(s^k_h))\leq O(\sqrt{S^2AH^4K'})$ holds for $h=1$. By induction, if the claim holds up to $h$, then we have
        \begin{align*}
            &\sum_{k=1}^{K'} \big(V^{\pi^k}_h(s^k_h)-V^*_h(s^k_h)\big)\\
            =&\sum_{k=1}^{K'} \big(\langle Q^{\pi^k}_h(s^k_h,\cdot),\pi^k_h(\cdot|s^k_h)\rangle - \langle Q^{*}_h(s^k_h,\cdot),\pi^*_h(\cdot|s^k_h)\rangle \big)\\
            \geq& \sum_{k=1}^{K'} \big(\langle Q^{\pi^k}_h(s^k_h,\cdot),\pi^k_h(\cdot|s^k_h)\rangle - \langle Q^{*}_h(s^k_h,\cdot),\pi^k_h(\cdot|s^k_h)\rangle\big)\\
            =&\sum_{k=1}^{K'} \langle P_h(s_h^k, \cdot)(V^{\pi^k}_{h+1}-V^*_{h+1}),\pi^k_h(\cdot|s^k_h)\rangle\\
            =&\sum_{k=1}^{K'} \big(V^{\pi^k}_{h+1}(s^k_{h+1})-V^*_{h+1}(s^k_{h+1})\big)+\underbrace{P_h(\cdot|s_h^k,\pi_h^{k}(s_h^k))(V^{\pi^k}_{h+1}-V^*_{h+1})(\cdot)-(V^{\pi^k}_{h+1}-V^*_{h+1})(s^k_{h+1})}_{(a)},
        \end{align*}
        where Term (a) is a martingale. Using Azuma-Hoeffding inequality, we have Term $(a)\leq O(\sqrt{H^2K'})$. Therefore by induction, for all $h\in[H]$ we have
        $$\sum_{k=1}^{K'} V^{\pi^k}_{h+1}(s^k_{h+1})-V^*_{h+1}(s^k_{h+1})\leq \sum_{k=1}^{K'} V^{\pi^k}_1(s^k_1)-V^*_1(s^k_1)+O(h\sqrt{H^2K'})\leq O(\sqrt{S^2AH^4K'}),$$
        which concludes the proof of Lemma \ref{lem:average:convergence}.
\end{proof}

\subsection{Useful Lemmas}
\subsubsection{Optimism} \label{sec:pf:optimism}
\begin{lemma}[Optimism at the beginning]\label{lem:opt_beg}
Conditioned on $F^c$, when $b_h^k(s,a)=\sqrt{\frac{2 \ln \frac{2 S A H T}{\delta^{\prime}}}{n_{h}^{k}(s, a)  }}+H\sqrt{\frac{4 S \ln \frac{3 S A H T}{\delta^{\prime}}}{n_{h}^{k}(s, a)  }}$, we have that
\begin{equation}
    -2b_h^k(s, a) \le Q_{h}^{k}(s, a)-c_{h}(s, a)-P_{h}(\cdot \mid s, a) V_{h+1}^{k}(\cdot) \leq 0.
\end{equation}
\end{lemma}
\begin{proof}
    Recall that $Q_h^k$ takes form 
    \begin{align} \label{eq:501}
        Q_h^k(s, a) = \max\{\bar{c}_h^k(s, a) + \bar{P}_h^{k}(\cdot \mid s, a)V_{h+1}^k(\cdot) - b_h^k(s, a), 0\}.
    \end{align}
    We have 
    \begin{align} \label{eq:502}
        &c_{h}(s, a) + P_{h}(\cdot \mid s, a) V_{h+1}^{k}(\cdot) - Q_{h}^{k}(s, a) \notag\\
    \le & c_{h}(s, a) - \bar{c}_h^k(s, a) + P_{h}(\cdot \mid s, a) V_{h+1}^{k}(\cdot) - \bar{P}_h^{k}(\cdot \mid s, a)V_{h+1}^k(\cdot) + b_h^k(s, a) .
    \end{align}
    Under the event $F^c$, we have 
    \begin{align} \label{eq:503}
        c_{h}(s, a) - \bar{c}_h^k(s, a) \le \sqrt{\frac{2 \ln \frac{2 S A H T}{\delta^{\prime}}}{n_{h}^{k}(s, a)  }},
    \end{align}
    and 
    \begin{align} \label{eq:504}
        &\left| P_{h}(\cdot \mid s, a) V_{h+1}^{k}(\cdot) - \bar{P}_h^{k}(\cdot \mid s, a)V_{h+1}^k(\cdot)\right| \notag\\
        \leq&\left\|\bar{P}_{h}^{k}\left(\cdot \mid s_{h}, a_{h}\right)- P_{h}\left(\cdot \mid s_{h}, a_{h}\right)\right\|_{1}\left\|V_{h+1}^{k}(\cdot)\right\|_{\infty} \notag\\
        \leq&H \cdot \left\|\bar{P}_{h}^{k}\left(\cdot \mid s_{h}, a_{h}\right)- P_{h}\left(\cdot \mid s_{h}, a_{h}\right)\right\|_{1} \notag\\
        \leq& H\sqrt{\frac{4 S \ln \frac{3 S A H T}{\delta^{\prime}}}{n_{h}^{k}(s, a)  }}.
    \end{align}
    Here the first inequality follows from Cauchy-Schwartz inequality, the second inequality uses the fact that $\|V_h^k(\cdot)\|_{\infty} \le H$ for any $(k, h) \in [K] \times [H]$, and the last inequality holds conditioned on the event $F^c$. Plugging \eqref{eq:503} and \eqref{eq:504} into \eqref{eq:502}, together with the definition of $b_h^k$, we obtain 
    \begin{align}  \label{eq:505}
        c_{h}(s, a) + P_{h}(\cdot \mid s, a) V_{h+1}^{k}(\cdot) - Q_{h}^{k}(s, a) \le 2b_h^k.
    \end{align}
    Similarly, by the definition of $Q_h^k$ in \eqref{eq:501}, we have 
    \begin{align}  \label{eq:506}
        &Q_{h}^{k}(s, a)-c_{h}(s, a)-P_{h}(\cdot \mid s, a) V_{h+1}^{k}(\cdot) \notag\\
       \le & \max\Big\{ \big(\bar{c}_h^k(s, a) - c_h(s, a) \big) + \big(\bar{P}_h^{k}(\cdot \mid s, a)V_{h+1}^k(\cdot) - {P}_h(\cdot \mid s, a)V_{h+1}^k(\cdot)\big) - b_h^k(s, a), 0 \Big\} \notag\\
       \le & \max\{b_h^k - b_h^k, 0\} = 0,
    \end{align}
    where the last inequality follows from \eqref{eq:503} and \eqref{eq:504}. Combining \eqref{eq:505} and \eqref{eq:506}, we finish the proof.
\end{proof}

\begin{lemma}\label{lem:coarseopt}
Conditioned on the event $F^c$, when the reference function $V^{\mathrm{ref}}$ satisfies $|V^{\mathrm{ref}}(s)-V^{*}(s)|\leq \frac{\sqrt{S^2AH^4}}{C_0}$ where $C_0=\sqrt{S^3AH^3}$, we have 
\begin{align*}
    -2b_h^k(s, a) \le Q_{h}^{k}(s, a)-c_{h}(s, a)-P_{h}(\cdot \mid s, a) V_{h+1}^{k}(\cdot) \leq 0.
\end{align*}
\end{lemma}
\begin{proof}
When $F$ does not happen,  we have
\begin{align*} 
    \left|c_{h}(s, a)-\bar{c}_{h}^{k}(s, a)\right| \leq \sqrt{\frac{2 \ln \frac{2 S A H T}{\delta^{\prime}}}{n_{h}^{k}(s, a)  }}
\end{align*}
and 
\begin{align*}
    \left| P_{h}(\cdot \mid s, a) V_{h+1}^{k}(\cdot) - \bar{P}_h^{k}(\cdot \mid s, a)V_{h+1}^k(\cdot) \right| \leq H\sqrt{\frac{4 S \ln \frac{3 S A H T}{\delta^{\prime}}}{n_{h}^{k}(s, a)  }}. 
\end{align*} 
Meanwhile, we have 
\begin{align} \label{eq:507}
    & \left| P_{h}(\cdot \mid s, a) V_{h+1}^{k}(\cdot) - \bar{P}_h^{k}(\cdot \mid s, a)V_{h+1}^k(\cdot) \right| \notag\\
    \leq & \left| \left(\bar{P}_{h}^{k}\left(\cdot \mid s, a\right)- P_{h}\left(\cdot \mid s, a\right)\right) V_{h+1}^{*}(\cdot) \right| + \left| \left(\bar{P}_{h}^{k}\left(\cdot \mid s, a\right) - P_{h}\left(\cdot \mid s, a\right)\right) (V_{h+1}^{k}-V_{h+1}^{\mathrm{ref}})(\cdot)\right| \notag\\ 
    &+\left|\left(\bar{P}_{h}^{k}\left(\cdot \mid s, a\right) - P_{h}\left(\cdot \mid s, a\right)\right) (V_{h+1}^{\mathrm{ref}}-V_{h+1}^{*})(\cdot)\right| 
\end{align}
For the first term, we have
\begin{align} \label{eq:421}
    &\left| \left(\bar{P}_{h}^{k}\left(\cdot \mid s, a\right)- P_{h}\left(\cdot \mid s, a\right)\right) V_{h+1}^{*} \right| \notag\\
    \le & \sqrt{\frac{2 \mathbb{V}_{Y \sim {P}_h (\cdot \mid s, a)}V_{h+1}^*(Y)\ln \left(\frac{2 SAHK}{\delta^{\prime}}\right)}{n_{h}^{k}(s, a)}}+\frac{7 \ln \left(\frac{2 SAHK}{\delta^{\prime}}\right)}{3 n_{h}^{k}(s, a)} \notag\\
    \le & \sqrt{\frac{4 \mathbb{V}_{Y \sim {P}_h (\cdot \mid s, a)}V_{h+1}^{\mathrm{ref}}(Y)\ln \left(\frac{2 SAHK}{\delta^{\prime}}\right)}{n_{h}^{k}(s, a)}}+ \sqrt{\frac{4H \ln \left(\frac{2 SAHK}{\delta^{\prime}}\right)}{S \cdot n_{h}^{k}(s, a)}}+\frac{7 \ln \left(\frac{2 SAHK}{\delta^{\prime}}\right)}{3 n_{h}^{k}(s, a)}, 
\end{align}
where the last inequality follows from the fact that $\sqrt{2\mathbb{V}(X)} \le 2\sqrt{\mathbb{V}(Y) + \mathbb{V}(Y - X)} \le 2\sqrt{\mathbb{V}(Y)} + 2\sqrt{\mathbb{V}(Y-X)}$ and $|V_{h+1}^* - V_{h+1}^{\mathrm{ref}}| \le \sqrt{H/S}$. Under the event $F^c$, we have 
\begin{align*}
    \left|\bar{P}_{h}^{k}(y \mid s, a)-P_{h}(y \mid s, a)\right| \le \sqrt{\frac{2\bar{P}_{h}^{k}(y \mid s, a)(1-\bar{P}_{h}^{k}(y \mid s, a)) \ln \left(\frac{2 SAHK}{\delta^{\prime}}\right)}{n_{h}^{k}(s, a)-1}}+\frac{7 \ln \left(\frac{2 SAHK}{\delta^{\prime}}\right)}{3 n_{h}^{k}(s, a)},
\end{align*}
By AM-GM inequality, we have
\begin{align*}
    \frac{1}{2}\bar{P}_{h}^{k}(y \mid s, a) + \frac{ \ln \left(\frac{2 S A H K}{\delta^{\prime}}\right)}{ n_{h}^{k}(s, a)} \geq \sqrt{\frac{2 \bar{P}_{h}^k(y \mid s, a) \ln \left(\frac{2 S A H K}{\delta^{\prime}}\right)}{n_{h}^{k}(s, a)}}\\
    \geq \sqrt{\frac{2 \bar{P}_{h}^k(y \mid s, a)\left(1-\bar{P}_{h}^k(y \mid s, a)\right) \ln \left(\frac{2 S A H K}{\delta^{\prime}}\right)}{n_{h}^{k}(s, a)-1}}
\end{align*}
which further implies that 
\begin{align*}
    P_{h}(y \mid s, a) - \frac{3}{2} \cdot \bar{P}_{h}^{k}(y \mid s, a) \le \frac{10 \ln \left(\frac{2 SAHK}{\delta^{\prime}}\right)}{3 n_{h}^{k}(s, a)}.
\end{align*}
Then, we have 
\begin{align} \label{eq:422}
    \mathbb{V}_{Y \sim {P}_h (\cdot \mid s, a)}V_{h+1}^{\mathrm{ref}}(Y) &= \sum_{y \in \cal{S}} {P}_h (y \mid s, a) \left(V_{h+1}^{\mathrm{ref}}(y) - P_h^k(\cdot \mid s, a)V_{h+1}^{\mathrm{ref}}(\cdot) \right)^2 \notag\\
    &\le \sum_{y \in \cal{S}} {P}_h (y \mid s, a) \left(V_{h+1}^{\mathrm{ref}}(y) - \bar{P}_h^k(\cdot \mid s, a)V_{h+1}^{\mathrm{ref}}(\cdot) \right)^2 \notag\\
    &\le \sum_{y \in \cal{S}} \left( \frac{3}{2}\bar{P}_h^k (y \mid s, a) + \frac{10 \ln \left(\frac{2 SAHK}{\delta^{\prime}}\right)}{3 n_{h}^{k}(s, a)}\right) \cdot  \left(V_{h+1}^{\mathrm{ref}}(y) - \bar{P}_h^k(\cdot \mid s, a)V_{h+1}^{\mathrm{ref}}(\cdot) \right)^2 \notag\\
    &\le \frac{3}{2} \mathbb{V}_{Y \sim \bar{P}_h^k (\cdot \mid s, a)}V_{h+1}^{\mathrm{ref}}(Y) + \frac{10SH^2 \ln \left(\frac{2 SAHK}{\delta^{\prime}}\right)}{3 n_{h}^{k}(s, a)}.
\end{align}
Plugging \eqref{eq:422} into \eqref{eq:421}, we obtain 
\begin{align} \label{eq:423}
    &\left| \left(\bar{P}_{h}^{k}\left(\cdot \mid s, a\right)- P_{h}\left(\cdot \mid s, a\right)\right) V_{h+1}^{*} \right| \notag\\
    \le & \sqrt{\frac{6 \mathbb{V}_{Y \sim \bar{P}_h^k (\cdot \mid s, a)}V_{h+1}^{\mathrm{ref}}(Y)\ln \left(\frac{2 SAHK}{\delta^{\prime}}\right)}{n_{h}^{k}(s, a)}}+ \sqrt{\frac{4H \ln \left(\frac{2 SAHK}{\delta^{\prime}}\right)}{S \cdot n_{h}^{k}(s, a)}}+\frac{8\sqrt{SH^2} \ln \left(\frac{2 SAHK}{\delta^{\prime}}\right)}{3 n_{h}^{k}(s, a)}.
\end{align}
Meanwhile, it holds that 
\begin{align} \label{eq:551}
    &\left| \left(\bar{P}_{h}^{k}\left(\cdot \mid s, a\right) - P_{h}\left(\cdot \mid s, a\right)\right) (V_{h+1}^{k}-V_{h+1}^{\mathrm{ref}})\right| \notag\\
    &\le  \sum_{y\in\mathcal{S}}\left| \bar{P}_{h}^{k}\left(y \mid s, a\right) - P_{h}\left(y \mid s, a\right)\right| \cdot |V^k_{h+1}(y)-V^{\mathrm{ref}}_{h+1}(y)|  \notag\\ 
    & \le \sum_{y\in\mathcal{S}}\left(\sqrt{\frac{2\bar{P}_{h}(y \mid s, a)(1-\bar{P}_{h}(y \mid s, a)) \ln \left(\frac{2 SAHK}{\delta^{\prime}}\right)}{n_{h}^{k}(s, a)-1}}+\frac{7 \ln \left(\frac{2 SAHK}{\delta^{\prime}}\right)}{3 n_{h}^{k}(s, a)}\right)|V^k_{h+1}(y)-V^{\mathrm{ref}}_{h+1}(y)|,
\end{align}
where the last inequality follows from the definition of event $F^c$. 
Plugging \eqref{eq:423} and \eqref{eq:551} into \eqref{eq:507}, we have 
\begin{align*}
    &\left| P_{h}(\cdot \mid s, a) V_{h+1}^{k}(\cdot) - \bar{P}_h^{k}(\cdot \mid s, a)V_{h+1}^k(\cdot) \right| \\
    \leq & \sqrt{\frac{6 \mathbb{V}_{Y \sim \bar{P}_h^k (\cdot \mid s, a)}V_{h+1}^{\mathrm{ref}}(Y)\ln \left(\frac{2 SAHK}{\delta^{\prime}}\right)}{n_{h}^{k}(s, a)}}+ \sqrt{\frac{4H \ln \left(\frac{2 SAHK}{\delta^{\prime}}\right)}{S \cdot n_{h}^{k}(s, a)}}+\frac{8\sqrt{SH^2} \ln \left(\frac{2 SAHK}{\delta^{\prime}}\right)}{3 n_{h}^{k}(s, a)} \notag\\
    &+\sum_{y\in\mathcal{S}}\left(\sqrt{\frac{2\bar{P}_{h}(y \mid s, a)(1-\bar{P}_{h}(y \mid s, a)) \ln \left(\frac{2 SAHK}{\delta^{\prime}}\right)}{n_{h}^{k}(s, a)-1}}+\frac{7 \ln \left(\frac{2 SAHK}{\delta^{\prime}}\right)}{3 n_{h}^{k}(s, a)}\right)|V^k_{h+1}(y)-V^{\mathrm{ref}}_{h+1}(y)| \\
    &+\underbrace{\sum_{y\in\mathcal{S}}\left|\bar{P}_{h}^{k}\left(\cdot \mid s, a\right)-P_{h}\left(\cdot \mid s, a\right)\right||V^{\mathrm{ref}}_{h+1}(y)-V^{*}_{h+1}(y)|}_{(b)}
\end{align*}
For Term (b), we divide $\mathcal{S}$ into two sets: $\mathcal{S}_0=\{y\in \mathcal{S} : n_{h}^{k}(y)\geq C_0\sqrt{k}\}$ and $\mathcal{S}_0^c$. Since $|V_{h+1}^{\mathrm{ref}}(s)-V_{h+1}^{*}(s)|\leq \frac{\sqrt{S^2AH^4}}{C_0}$, we have
\begin{align*}
    &\sum_{y\in\mathcal{S}_0}\left|\bar{P}_{h}^{k}\left(\cdot \mid s, a\right)-P_{h}\left(\cdot \mid s_{h}, a_{h}\right)\right| |V^{\mathrm{ref}}_{h+1}(y)-V^{*}_{h+1}(y)|\\
    \leq & \sqrt{\frac{4 S \ln \frac{3 S A H T}{\delta^{\prime}}}{n_{h}^{k}(s, a)  }}\frac{\sqrt{S^2AH^4}}{C_0}
\end{align*}
For $y\in \mathcal{S}_0^c$, we have 
\begin{align*}
    &\sum_{y\in\mathcal{S}_0^c}\left|\bar{P}_{h}^{k}\left(\cdot \mid s, a\right)-P_{h}\left(\cdot \mid s_{h}, a_{h}\right)\right||V^{\mathrm{ref}}_{h+1}(y)-V^{*}_{h+1}(y)|\\
    \leq &\sum_{y\in\mathcal{S}_0^c}\left(\sqrt{\frac{2\bar{P}_{h}^k(s^{\prime} \mid s, a)(1-\bar{P}_{h}^k(s^{\prime} \mid s, a)) \ln \left(\frac{2 SAHK}{\delta^{\prime}}\right)}{n_{h}^{k}(s, a)-1}}+\frac{7 \ln \left(\frac{2 SAHK}{\delta^{\prime}}\right)}{3 n_{h}^{k}(s, a)}\right)|V^{\mathrm{ref}}_{h+1}(y)-V^{*}_{h+1}(y)|\\
    \leq &\sum_{y\in\mathcal{S}_0^c}\left(\sqrt{\frac{4\bar{P}_{h}^k(s^{\prime} \mid s, a)\ln \left(\frac{2 SAHK}{\delta^{\prime}}\right)}{n_{h}^{k}(s, a)}}\right)|V^{\mathrm{ref}}_{h+1}(y)-V^{*}_{h+1}(y)|+\frac{7 \ln \left(\frac{2 SAHK}{\delta^{\prime}}\right)}{3 n_{h}^{k}(s, a)}|V^{\mathrm{ref}}_{h+1}(y)-V^{*}_{h+1}(y)|\\
    = & \sum_{y\in\mathcal{S}_0^c}\sqrt{\frac{4n^{k}_h(s,a,y)\ln \left(\frac{2 SAHK}{\delta^{\prime}}\right)}{n^{k}_h(s,a)^2}}|V^{\mathrm{ref}}_{h+1}(y)-V^{*}_{h+1}(y)|+\frac{7SH \ln \left(\frac{2 SAHK}{\delta^{\prime}}\right)}{3 n_{h}^{k}(s, a)}\\
    \leq & \sum_{y\in\mathcal{S}_0^c}\sqrt{\frac{4n^{k}_h(y)\ln \left(\frac{2 SAHK}{\delta^{\prime}}\right)}{n^{k}_h(s,a)^2}}H+\frac{7SH \ln \left(\frac{2 SAHK}{\delta^{\prime}}\right)}{3 n_{h}^{k}(s, a)}\\
    \leq & \sum_{y\in\mathcal{S}_0^c}\sqrt{\frac{4C_0\sqrt{K}\ln \left(\frac{2 SAHK}{\delta^{\prime}}\right)}{n^{k}_h(s,a)^2}}H+\frac{7SH \ln \left(\frac{2 SAHK}{\delta^{\prime}}\right)}{3 n_{h}^{k}(s, a)}\\
    \leq & \frac{2S\sqrt{C_0}K^{1/4}H\sqrt{\ln \left(\frac{2 SAHK}{\delta^{\prime}}\right)}}{n_h^{k}(s,a)}+\frac{7SH \ln \left(\frac{2 SAHK}{\delta^{\prime}}\right)}{3 n_{h}^{k}(s, a)}
\end{align*}
Therefore, we have
\begin{align*}
    &\left| \left(\bar{P}_{h}^{k}\left(\cdot \mid s_{h}, a_{h}\right)- P_{h}\left(\cdot \mid s_{h}, a_{h}\right)\right) V_{h+1}^{k} \right|\\
    \leq& \sqrt{\frac{6 \mathbb{V}_{Y \sim \bar{P}_h^k (\cdot \mid s, a)}V_{h+1}^{\mathrm{ref}}(Y)\ln \left(\frac{2 SAHK}{\delta^{\prime}}\right)}{n_{h}^{k}(s, a)}}+ \sqrt{\frac{4H \ln \left(\frac{2 SAHK}{\delta^{\prime}}\right)}{S \cdot n_{h}^{k}(s, a)}}+\frac{8\sqrt{SH^2} \ln \left(\frac{2 SAHK}{\delta^{\prime}}\right)}{3 n_{h}^{k}(s, a)}\\
    &+\sum_{y\in\mathcal{S}}\left(\sqrt{\frac{2\bar{P}_{h}^k(s^{\prime} \mid s, a)(1-\bar{P}_{h}^k(s^{\prime} \mid s, a)) \ln \left(\frac{2 SAHK}{\delta^{\prime}}\right)}{n_{h}^{k}(s, a)-1}}+\frac{7 \ln \left(\frac{2 SAHK}{\delta^{\prime}}\right)}{3 n_{h}^{k}(s, a)}\right)|V^k_{h+1}(y)-V^{\mathrm{ref}}_{h+1}(y)|\\
    &+\sqrt{\frac{4 S \ln \frac{3 S A H T}{\delta^{\prime}}}{n_{h}^{k}(s, a)  }}\frac{\sqrt{S^2AH^4}}{C_0}+\frac{2S\sqrt{C_0}K^{1/4}H\sqrt{\ln \left(\frac{2 SAHK}{\delta^{\prime}}\right)}}{n_h^{k}(s,a)}+\frac{7SH \ln \left(\frac{2 SAHK}{\delta^{\prime}}\right)}{3 n_{h}^{k}(s, a)}
\end{align*}  

By definition, $Q_h^k(s, a) = \max\{\bar{c}_h^k(s, a) + \bar{P}_h^{k}(\cdot \mid s, a)V_{h+1}^k(\cdot) - b_h^k(s, a), 0\}$. Then we can prove that \eqref{eq:505} and \eqref{eq:506} still hold, which finishes the proof.
\end{proof}

\subsubsection{Mirror Descent}
The mirror descent (MD) algorithm \citep{beck2003mirror} is a proximal convex optimization method that minimizes a linear approximation of the objective together with a proximity term, defined in terms of a Bregman divergence between the old and new solution estimates. In our analysis we choose the Bregman divergence to be the $l_2$ norm. If $\left\{f_{k}\right\}_{k=1}^{K}$ is a sequence of convex functions $f_{k}: \mathbb{R}^{d} \rightarrow \mathbb{R}$ and $C$ is a constraints set, the $k$-th iterate of $\mathrm{MD}$ is the following:
\begin{equation}
x_{k+1} \in \underset{x \in C}{\arg \min }\left\{\eta_{k}\left\langle g_{k}\left(x_{k}\right), x-x_{k}\right\rangle+\left\|x - x_{k}\right\|_2^2\right\},
\end{equation}
where $\eta_k$ is the stepsize. The MD algorithm ensures $\operatorname{Regret}\left(K^{\prime}\right)=\sum_{k=1}^{K^{\prime}} f\left(x_{k}\right)-$ $\min _{x} f(x) \le O(\sqrt{K'})$ for all $K^{\prime} \in[K]$.

The following lemma (Theorem 6.8 in \cite{orabona2019modern}) is a fundamental inequality for analysis of OMD regret, which will be used in our analysis.
\begin{lemma}[OMD Regret, Theorem 6.8 in \cite{orabona2019modern}]\label{lem:omdregret}
Assume $\eta_{k+1} \leq \eta_{k}, k=1, \ldots, K$. Then, using OMD with the $l_2$ norm, learning rate $\{\eta_k\}$ and uniform initialization $x_{1}=[1 / d, \ldots, 1 / d]$, the following regret bounds hold
\begin{equation}\label{eqn:omdineq}
\sum_{t=1}^{T}\left\langle\boldsymbol{g}_{t}, \boldsymbol{x}_{t}-\boldsymbol{u}\right\rangle \leq \max _{1 \leq t \leq T} \frac{B_{\psi}\left(\boldsymbol{u} ; \boldsymbol{x}_{t}\right)}{\eta_{T}}+\frac{1}{2} \sum_{t=1}^{T} \eta_{t}\left\|\boldsymbol{g}_{t}\right\|_{2}^{2}.
\end{equation}
\end{lemma}

In our analysis, by adapting the above lemma to our notation, we get the following lemma.
\begin{lemma}[OMD in Policy Optimization]\label{lem:omd_in_po}
Assume $\eta_{k+1} \leq \eta_{k}, k=1, \ldots, K$. Then, using OMD with the $l_2$ norm, learning rate $\{\eta_k\}$ and uniform initialization $\pi_{h}^1(\cdot \mid s)=[1 / A, \ldots, 1 / A]$, the following regret bounds hold
\begin{equation}
\sum_{k=1}^{K}\left\langle Q_{h}^{k}(s,\cdot), \pi_{h}^{k}(\cdot \mid s)-\pi_{h}(\cdot \mid s)\right\rangle \leq \frac{2}{\eta_{K}}+\frac{1}{2} \sum_{k=1}^{K}\eta_k \sum_{a}\left(Q_{h}^{k}(s, a)\right)^{2}.
\end{equation}
Moreover, if we choose $\eta_k=1/\sqrt{AH^2k}$, we have
\begin{equation}
\sum_{k=1}^{K}\sum_{h = 1}^H\left\langle Q_{h}^{k}(s,\cdot), \pi_{h}^{k}(\cdot \mid s)-\pi_{h}(\cdot \mid s)\right\rangle \leq 3\sqrt{AH^4K}.
\end{equation}
\end{lemma}
\begin{proof}
Fix $h \in [H]$. Replace $\boldsymbol{g}_{t}$ with $Q_{h}^{k}(s,\cdot)$ and $\boldsymbol{x}_{t}$ with $\pi_h^{k}(\cdot\mid s)$ in \eqref{eqn:omdineq}. Since $\|\pi_{h}^{*}-\pi_{h}^{k}\|_2^2\leq 2$ and $Q_{h}^{k}(s, a)\leq H$, we have
\begin{align*}
\sum_{k=1}^{K}\left\langle Q_{h}^{k}(s,\cdot), \pi_{h}^{k}(\cdot \mid s)-\pi_{h}(\cdot \mid s)\right\rangle &\leq \frac{2}{\eta_{K}}+\frac{1}{2} \sum_{k=1}^{K}\eta_k \sum_{a}\left(Q_{h}^{k}(s, a)\right)^{2}\\
&\leq \frac{2}{\eta_{K}}+\frac{1}{2} \sum_{k=1}^{K}\eta_k AH^2.
\end{align*}
Let $\eta_k=1/\sqrt{AH^2k}$, we have $\sum_{k=1}^{K}\eta_{k}\leq \sqrt{\frac{4K}{AH^2}}$, which further implies that 
\begin{align*}
\sum_{k=1}^{K}\left\langle Q_{h}^{k}(s,\cdot), \pi_{h}^{k}(\cdot \mid s)-\pi_{h}(\cdot \mid s)\right\rangle \leq 3\sqrt{AH^2K}
\end{align*}
for any $h \in [H]$. Taking summation over $h \in [H]$ concludes our proof. 
\end{proof}

\subsection{Sum of Bonus}
\begin{lemma}\label{lem:sumbonus}
    It holds that 
\begin{equation}
    \sum_{k=1}^K\sum_{h=1}^H b^k_h(s^k_h,a^k_h) \leq \tilde{O}(\sqrt{SAH^3K}+S^{5/2}A^{5/4}H^{11/4}K^{1/4}).
\end{equation}
\end{lemma}
\begin{proof}
For simplicity we let $n^k_h = n^k_h(s^k_h,a^k_h), \bar{P}^k_h(\cdot) = \bar{P}_h^k(\cdot \mid s^k_h,a^k_h), P_h(\cdot) = P_h^k(\cdot \mid s^k_h,a^k_h)$. By the same arguments in \eqref{eq:421}, \eqref{eq:422} and \eqref{eq:423}, we have 
\begin{align*}
    &\sqrt{\frac{6 \mathbb{V}_{Y \sim \bar{P}_h^k (\cdot \mid s, a)}V_{h+1}^{\mathrm{ref}}(Y)\ln \left(\frac{2 SAHK}{\delta^{\prime}}\right)}{n_{h}^{k}(s, a)}}+ \sqrt{\frac{4H \ln \left(\frac{2 SAHK}{\delta^{\prime}}\right)}{S \cdot n_{h}^{k}(s, a)}}+\frac{8\sqrt{SH^2} \ln \left(\frac{2 SAHK}{\delta^{\prime}}\right)}{3 n_{h}^{k}(s, a)} \\
    \le & \tilde{O}\left( \sqrt{\frac{ \mathbb{V}_{Y \sim {P}_h^k (\cdot \mid s, a)}V_{h+1}^{*}(Y)}{n_{h}^{k}(s, a)}}+ \sqrt{\frac{H}{S \cdot n_{h}^{k}(s, a)}} + \frac{\sqrt{SH^2} }{3 n_{h}^{k}(s, a)} \right).
\end{align*}
Then, by the definition of $b^k_h$,
\begin{align}
    &\sum_{k=1}^K\sum_{h=1}^H b^k_h(s^k_h,a^k_h)\notag\\
    \leq&\sum_{k=1}^K\sum_{h=1}^H \tilde{O}\left(\sqrt{\frac{ \mathbb{V}_{Y \sim {P}_h^k (\cdot \mid s, a)}V_{h+1}^{*}(Y)}{n_{h}^{k}(s, a)}} +\sum_{y\in\mathcal{S}}\sqrt{\frac{\bar{P}_h^k(y)}{n^k_h}}|V^k_{h+1}(y)-V^{\mathrm{ref}}_{h+1}(y)|+\frac{S^{3/2}A^{1/4}H^{7/4}K^{1/4}}{n^k_h}\right)\notag\\
    \leq &\tilde{O}(\sqrt{SAH^3K})+\sum_{k=1}^K\sum_{h=1}^H\sum_{y\in\mathcal{S}}\tilde{O}\left(\sqrt{\frac{\bar{P}_h^k(y)}{n^k_h}}|V^k_{h+1}(y)-V^{\mathrm{ref}}_{h+1}(y)|\right)+\tilde{O}(S^{5/2}A^{5/4}H^{11/4}K^{1/4})\notag\\
    \leq&\tilde{O}(\sqrt{SAH^3K})+\sum_{k=1}^K\sum_{h=1}^H\sum_{y\in\mathcal{S}}\tilde{O}\left(\sqrt{\frac{P_h^k(y)+O(\sqrt{1/n^k_h})}{n^k_h}}|V^k_{h+1}(y)-V^{\mathrm{ref}}_{h+1}(y)|\right)+\tilde{O}(S^{5/2}A^{5/4}H^{11/4}K^{1/4})\notag\\
    \leq&\tilde{O}(\sqrt{SAH^3K})+\underbrace{\sum_{k=1}^K\sum_{h=1}^H\sum_{y\in\mathcal{S}}\tilde{O}\left(\sqrt{\frac{P_h^k(y)}{n^k_h}}|V^k_{h+1}(y)-V^{\mathrm{ref}}_{h+1}(y)|\right)}_{(v)}+\tilde{O}(S^{5/2}A^{5/4}H^{11/4}K^{1/4})\label{eqn:v1}
\end{align}
The second inequality follows from Lemma 19 in \citet{zhang2020almost} and standard techniques (\textit{e.g.} \cite{shani2020optimistic} or \cite{zhang2020almost}). Using Lemma \ref{lem:sumbonus2}, we can further bound $\text{Term}\ (v)$ as:
\begin{equation*}
    \text{Term}\ (v) \leq \tilde{O}(\sqrt{SAH^3K})+\tilde{O}(S^{5/2}A^{5/4}H^{3/2}K^{1/4})
\end{equation*}
Therefore, combining Equation \eqref{eqn:v3} and \eqref{eqn:v1}, we complete the proof of Theorem \ref{thm:fine_reg}.
\end{proof}

\begin{lemma}\label{lem:sumbonus2}
\begin{align*}
    \sum_{k=1}^K\sum_{h=1}^H\sum_{y\in\mathcal{S}}\left(\sqrt{\frac{P_h^k(y)}{n^k_h}}|V^k_{h+1}(y)-V^{\mathrm{ref}}_{h+1}(y)|\right)\leq \tilde{O}(\sqrt{SAH^3K})+\tilde{O}(S^{5/2}A^{5/4}H^{3/2}K^{1/4})
\end{align*}
\end{lemma}
\begin{proof}
We have
\begin{align}
    &\sum_{k=1}^K\sum_{h=1}^H\sum_{y\in\mathcal{S}}\left(\sqrt{\frac{P_h^k(y)}{n^k_h}}|V^k_{h+1}(y)-V^{\mathrm{ref}}_{h+1}(y)|\right)\notag\\
    \leq& \sum_{k=1}^K\sum_{h=1}^H\big(\sum_{y\in\mathcal{S}}P_h^k(y)\sqrt{\frac{1}{P_h^k(y)n^k_h}}|V^k_{h+1}(y)-V^{\mathrm{ref}}_{h+1}(y)|-\sqrt{\frac{1}{P_h^k(s^k_{h+1})n^k_h}}|V^k_{h+1}(s^k_{h+1})-V^{\mathrm{ref}}_{h+1}(s^k_{h+1})|\big)\notag\\
    &+\sum_{k=1}^K\sum_{h=1}^H\sqrt{\frac{1}{P_h^k(s^k_{h+1})n^k_h}}|V^k_{h+1}(s^k_{h+1})-V^{\mathrm{ref}}_{h+1}(s^k_{h+1})|\notag\\
    \leq& \tilde{O}(\sqrt{H^3K})+\sum_{k=1}^K\sum_{h=1}^H\sqrt{\frac{1}{P_h^k(s^k_{h+1})n^k_h}}|V^k_{h+1}(s^k_{h+1})-V^{\mathrm{ref}}_{h+1}(s^k_{h+1})|\notag\\
    =&\sum_{k=1}^K\sum_{h=1}^H\big(\sqrt{\frac{1}{P_h^k(s^k_{h+1})n^k_h}}\underbrace{|V^k_{h+1}(s^k_{h+1})-V^{*}_{h+1}(s^k_{h+1})|}_{\text{the sum}\ \leq\  \tilde{O}(\sqrt{H^4S^2AK})}+\sqrt{\frac{1}{P_h^k(s^k_{h+1})n^k_h}}\underbrace{|V^{*}_{h+1}(s^k_{h+1})-V^{\mathrm{ref}}_{h+1}(s^k_{h+1})|}_{\text{almost}\ \leq \ \tilde{O}(1/\sqrt{S})}\big)\notag\\
    &+\tilde{O}(\sqrt{H^3K})\notag\\
    \leq & \tilde{O}(S^2AH\sqrt{\sqrt{H^2S^2AK}})+\tilde{O}(\sqrt{S^2AH^2K}/\sqrt{S})+\tilde{O}(\sqrt{H^3K})\notag\\
    = &\tilde{O}(\sqrt{SAH^3K})+\tilde{O}(S^{5/2}A^{5/4}H^{3/2}K^{1/4})\label{eqn:v2}
\end{align}
The second line is the sum of martingale differences bounded by $H$, directly applying Azuma-Hoeffding inequality yields the second inequality. The third inequality makes use of Lemma 11 in \cite{zhang2020almost}. 
\paragraph{The third inequality in (44)} This is a technique used in (Azar et al. 2017). We explain the main idea here.

First we define the typical state-actions pairs as
$$
[y]_{k,h} \stackrel{\text { def }}{=}\left\{y: y \in \mathcal{S}, n_{h}^k(s_h^k, a_h^k) P(y \mid s_h^k, a_h^k) \geq 2 H^{2}S L\right\}
$$
which means these state-action pairs are visited frequently enough. Define $\widetilde{\Delta}_{h+1}^k(y) = \left|V_{h+1}^{k}(y)-V_{h+1}^{\mathrm{ref}}(y)\right|$. We have
\begin{align*}
    &\sum_{y \in \mathcal{S}} \sqrt{\frac{P_{h}^k(y)}{n_{h}^k}} \widetilde{\Delta}_{h+1}^k(y) \\
    =&\sum_{y \in[y]_{k,h}} \sqrt{\frac{P_{h}^k(y)}{n_{h}^k}} \widetilde{\Delta}_{h+1}^k(y)
    +\sum_{y \notin[y]_{k,h}} \sqrt{\frac{P_{h}^k(y)}{n_{h}^k}} \widetilde{\Delta}_{h+1}^k(y) .
\end{align*}
The second term can be bounded by
$$
\sum_{y \notin[y]_{k,h}} \sqrt{\frac{P_{h}^k(y) n_{h}^k}{{n_{h}^k}^{2}}} \widetilde{\Delta}_{h+1}^k(y) \leq \frac{S H \sqrt{4 L H^{2}S}}{n_{h}^k}
$$
So we only have to deal with the first term, in which $P_h^k(y)n_h^k$ is large. Thus the martingale can be bounded.

\paragraph{The last inequality in (44)} As above, we only need to consider the case where $P_h^k(s_{h+1}^k)n_h^k \geq 2H^2SL$ ($n^k_h$ is the shorthand for $n^k_h(s^k_h,a^k_h)$). Then we know the first term is bounded by $\tilde{O}(\sqrt{SAH^3K})$. For the second term, we need to use the Multiplicative Chernoff bound (See the Wiki for Chernoff bound):
$\operatorname{Pr}(X \leq(1-\delta) \mu) \leq e^{-\frac{\delta^{2} \mu}{2}},$
where $0 \leq \delta \leq 1$ and $X = \sum^n_{i=1}X_i$, $X_i$ are independent Bernoulli r.v., $\mathbb{E}[X] = \mu$. Set $X = n^k_h(s^k_h,a^k_h,s^k_{h+1}),\ \mu = P_h^k(s_{h+1}^k)n_h^k$. Taking union bound over all $h,k,\mathcal{S}$ with $P_h^k(s_{h+1}^k)n_h^k \geq 2H^2SL$, we have that with high probability, $P_h^k(s^k_{h+1})n_h^k \geq \frac{1}{2} n_h^k(s_h^k, a_h^k, s^k_{h+1})$. With this we can now apply Lemma 11 in (Zhang et al. 2020).
\end{proof}

\section{Explanation of How UCBVI Uses Optimism}\label{sec:explanation}
In \citet{azar2017minimax}, they need to bound the term $(\bar{P}_{h}^k-P_{h})(\cdot \mid s,a)(V^{k}_{h+1} - V_{h+1}^*)(\cdot)$ using optimism, as mentioned in Section \ref{sec:technique}. 

Define $\Delta_{h}^k \stackrel{\text { def }}{=} V_{h}^{*}-V_{h}^{\pi^{k}}$, $\widetilde{\Delta}_{h}^k \stackrel{\text { def }}{=} V_{h}^k-V_{h}^{\pi^{k}}$, and $\widetilde{\delta}_{h}^{k} \stackrel{\text { def }}{=} \widetilde{\Delta}_{h}^{k}\left(s_{h}^{k}\right)$. We denote by $\square$ a numerical constant which can vary from line to line. We also use $L$ to represent the logarithmic term $L=\ln (\square H S A T / \delta)$.

Using Bernstein's inequality, this term is bounded by
$$
\sum_{y} P^{\pi^{k}}\left(y \mid s_{h}^{k}\right) \sqrt{\frac{\square L}{P^{\pi^{k}}\left(y \mid s_{h}^k\right) n_h^k}} \Delta_{ h+1}^k(y)+\frac{\square S H L}{n_h^k}.
$$
where $n_h^k \stackrel{\text { def }}{=} n_{k}\left(s_{h}^k, \pi^{k}\left(s_h^k\right)\right)$. Now considering only the $y$ such that $P^{\pi^{k}}\left(y \mid s_h^k\right) n_h^k \geq \square H^{2} L$, and since $0 \leq \Delta_{h+1}^k \leq \widetilde{\Delta}_{h+1}^k$ \textbf{by optimism}, then $\left(\widehat{P}_{k}^{\pi_{k}}-P^{\pi_{k}}\right) \Delta_{k, h+1}\left(s_h^k\right)$ is bounded by
$$
\bar{\epsilon}_h^k+\sqrt{\frac{\square L}{P^{ \pi^{k}}\left(s_{h+1}^k \mid s_h^k\right) n_h^k}} \widetilde{\delta}_{k, h+1}+\frac{\square S H L}{n_h^k} \leq \bar{\epsilon}_h^k+\frac{1}{H} \widetilde{\delta}_{h+1}^k+\frac{\square S H L}{n_h^k}.
$$
where $ \bar{\epsilon}_h^k \stackrel{\text { def }}{=} \sqrt{\frac{\square L}{n_h^k}}\left(\sum_{y} P^{\pi^{k}}\left(y \mid s_h^k\right) \frac{\widetilde{\Delta}_{h+1}^k(y)}{\sqrt{P^{\pi^{k}} \left(y \mid s_h^k\right)}}-\frac{\widetilde{\delta}_{h+1}^k}{\sqrt{P^{\pi_{k}}\left(s_{ h+1}^k \mid s_h^k\right)}}\right) .$
The sum over the neglected $y$ such that $P^{\pi^{k}}\left(y \mid s_h^k\right) n_h^k<\square H^{2} L$ contributes to an additional term
$$
\sum_{y} \sqrt{\frac{\square P^{\pi^{k}}\left(y \mid s_h^k\right) n_h^k L}{{n_h^k}^{2}}} \Delta_{h+1}^k(y) \leq \frac{\square S H^{2} L}{n_h^k}.
$$
Then they prove that the sum of $\bar{\epsilon}_h^k$ is of order $\tilde{O}(\sqrt{T})$ and the sum of $\frac{1}{n_h^k}$ is a constant order term.

\end{document}